\documentclass[twoside,11pt]{article}
\usepackage{amsfonts}
\usepackage{amsgen,amsmath,amstext,amsbsy,amsopn,amsthm,stmaryrd}

\DeclareMathAlphabet\mathbfcal{OMS}{cmsy}{b}{n}

\newcommand{\bbR}{\mathbb{R}}
\newcommand{\bbN}{\mathbb{N}}

\newcommand{\cS}{\mathcal{S}}

\newcommand{\overbar}[1]{\mkern 1.5mu\overline{\mkern-1.5mu#1\mkern-1.5mu}\mkern 1.5mu}

\newcommand{\bfu}{{\mathbf{u}}}
\newcommand{\bfv}{{\mathbf{v}}}

\newcommand{\bfx}{{\mathbf{x}}}

\newcommand{\z}{{\mathbf{z}}}
\newcommand{\X}{{\mathbf{X}}}

\newcommand{\I}{{\mathbf{I}}}

\newcommand{\cM}{{\mathcal{M}}}
\newcommand{\range}{\text{range}}

\newcommand{\bfV}{{\mathbf{V}}}

\newcommand{\Z}{{\mathbf{Z}}}

\newcommand{\iso}{\text{iso}}

\newcommand{\cL}{\mathcal{L}}

\newcommand{\argmin}{\mathop{\rm arg\min}}

\newcommand{\bbE}{\mathbb{E}}

\newcommand{\ie}[0]{\emph{i.e., }}

\newcommand{\eg}[0]{\emph{e.g., }}

\usepackage{hyperref}
\usepackage[capitalise]{cleveref}
\usepackage[preprint,nohyperref]{jmlr2e}
\graphicspath{{./figures/}}
\usepackage[ruled]{algorithm2e}
\crefname{algocf}{alg.}{algs.}
\Crefname{algocf}{Algorithm}{Algorithms}
\usepackage{mathtools}
\usepackage{subcaption}
\usepackage{enumerate}
\usepackage{empheq}
\usepackage{comment}
\usepackage{enumitem}
\usepackage{booktabs}
\usepackage{blkarray}
\usepackage{algpseudocode}
\usepackage{url}
\usepackage{longtable}
\usepackage{multirow}
\usepackage{bbm}
\usepackage[usenames]{color}

\usepackage{tikz}
\usepackage{adjustbox}

\jmlrheading{1}{2021}{ }{ }{ }{paper id}{Yue Gao and Garvesh Raskutti}

\ShortHeadings{Improved Prediction and Network Estimation Using the Monotone SIMAM}{Gao and Raskutti}
\firstpageno{1}

\title{Improved Prediction and Network Estimation Using the Monotone Single Index Multi-variate Autoregressive Model}

\author{\name Yue\ Gao \email ygao266@wisc.edu \\
       \addr Department of Statistics\\
       University of Wisconsin Madison\\
       Madison, WI 53703, USA
       \AND
       \name Garvesh \ Raskutti \email garvesh@gmail.com \\
       \addr Department of Statistics\\
       University of Wisconsin Madison\\
       Madison, WI 53703, USA
       }

\editor{ }

\begin{document}
\maketitle

\begin{abstract}
Network estimation from multi-variate point process or time series data is a problem of fundamental importance. Prior work has focused on parametric approaches that require a known parametric model, which makes estimation procedures less robust to model mis-specification, non-linearities and heterogeneities. In this paper, we develop a semi-parametric approach based on the monotone single-index multi-variate autoregressive model (SIMAM) which addresses these challenges. We provide theoretical guarantees for dependent data, and an alternating projected gradient descent algorithm (based on~\cite{dai2021convergence}). Significantly we do not explicitly assume mixing conditions on the process (although we do require conditions analogous to restricted strong convexity) and we achieve rates of the form $O(T^{-\frac{1}{3}} \sqrt{s\log(TM)})$ (optimal in the independent design case) where $s$ is the threshold for the maximum in-degree of the network that indicates the sparsity level, $M$ is the number of actors and $T$ is the number of time points. 
In addition, we demonstrate the superior performance both on simulated data and two real data examples where our SIMAM approach out-performs state-of-the-art parametric methods both in terms of prediction and network estimation.
\end{abstract}

\section{Introduction}
   
Multi-variate time series or point process data arises in a number of settings such as social networks~\cite{zhou2013learning,richey2008autoregressive,mark2019estimating,Raginsky_2012}, crime networks~\cite{stomakhin2011reconstruction,Mark_2019,egesdal2010statistical}, electrical systems~\cite{ertekin2015reactive}, neuroscience~\cite{brown2004multiple,hall2015online,smith2003estimating,fujita2007modeling} and many others. One of the questions of interest in multi-variate time series/point process data is estimating an \emph{influence network} which captures the temporal influence amongst different nodes.
For instance, in a social network, different nodes represent different individuals or media sources, whose behaviours, such as posting articles or reporting hot events, can be observed through time. By investigating such time-stamped data, we seek to discover the flow of information or potential communities through the inference of the underlying influence network.

There is a large body of recent work on parametric models and estimators for learning influence networks (see e.g.~\cite{hall2016inference, Mark_2019}). In many scenarios, underlying non-linearities and heterogeneities make it difficult to posit a parametric model. Furthermore, parametric models often do not yield good prediction performance due to their lack of flexibility and inability to model non-linearities. In this work, we use a semi-parametric network estimation approach that addresses these challenges. In particular, rather than using standard parametric approaches, we use the monotone single index model (SIM) for network estimation. 

The monotone single index model (MSIM) has been widely used in many settings~\cite{foster2013variable,balabdaoui2019least, groeneboom2019estimation}. The semi-parametric construction allows the interpretation using the ``parametric'' part while the ``non-parametric'' part allows the flexibility to model non-linearities and misspecified link functions. Typically, the non-parametric function is assumed to be monotone, which covers a number of interesting examples, including all generalized linear models and many other examples and this function does not need to be pre-specified. From a statistical and algorithmic perspective, MSIM in high dimensions presents a number of technical challenges (see e.g.~\cite{chen2014generalised,foster2013variable}). Many of these have been addressed in settings where we have independent samples~\cite{foster2013variable,balabdaoui2019least}.

In the context of multi-variate time series and point process models, the MSIM provides a natural semi-parametric framework by modelling each time series as a separate MSIM, where the features/co-variates are the data from previous time points.
One of the major theoretical challenges with applying the MSIM to autoregressive point processes is providing theoretical guarantees while accounting for the complex nonlinear dependence. 

This paper addresses this challenge by providing an alternating PGD algorithm and its theoretical guarantees for the monotone single index multi-variate autoregressive model (SIMAM). Significantly, we do not explicitly assume any mixing condition on the multi-variate time series as is done in~\cite{zhou2018non}, although we do require conditions analogous to restricted strong convexity. We also support our theoretical findings by giving empirical evidence through simulations and real data examples, illustrating the superior performance for monotone SIMAM in terms of prediction and variable selection compared to existing state-of-the-art approaches.

\subsection{Related works}
Various parametric approaches have been widely explored in a large body of prior work to learn the influence network from multi-variate time series or point process data. One standard approach is the vector autoregressive (VAR) model
\cite{lutkepohl2013vector,canova1995vector,2008ssvar_lasso}. 
To avoid the limitations of VAR in non-Gaussian or non-linear autoregressive processes, vector generalized linear autoregressive (GLAR) model 
\cite{hall2016inference,dunsmuir2015generalized,hall2018learning,shephard1995generalized}
is proposed and widely used as an extension of VAR, in which non-linear structure is introduced by a known link function according to prior knowledge. 
By specifying the link function and the conditional distribution within the exponential family, GLAR can be adjusted to many specific models, such as the Bernoulli autoregressive model \cite{p2019highdimensional} and log-linear Poisson autoregressive (PAR) model
\cite{fokianos2009poisson,zhu2011estimation}.
Although these models are more flexible, they are still restricted to non-linear models with known and fixed parametrization, which may not be applicable to real-world settings.

In order to improve the robustness and flexibility of parametric autoregressive models in multi-variate time series, non-parametric approaches have been explored and developed
(see e.g. \cite{scaillet2004nonparametric,hardle1992kernel}).
For instance, in the recent work \cite{zhou2018non}, a non-parametric additive autoregressive network model is developed, involves replacing linear terms with additive functions belonging to a reproducing kernel Hilbert space (RKHS). The estimators are obtained through a penalized maximum likelihood procedure.

In this work, we don't directly impose any smoothness assumptions as we formulate the conditional expectations directly through an isotonic single index autoregressive model for a $M$-dimensional multi-variate time series, i.e.
\begin{equation}
    \bbE(X_{t,j}|X_{t-1}) = f_j^{*}(X_{t-1}^T u_j^{*}), ~j\in [M],
\end{equation}
where $f_j^*$ is an unknown link function and $u_j^*$ is the direction or the index to be estimated as one column of the influence network. We approximate the network parameters by minimizing the mean squared loss rather than maximizing the unknown likelihood function. 

In a non-parametric regression setting, the single index model has been well developed over the past decades. 
Classical approaches in estimating the single index model include profile likelihoods and smooth kernels
\cite{carroll1997generalized, xue2006empirical,hristache2001direct, wang2010estimation, naik2001single}.
For multi-variate time series data, \cite{TracyZWu2011,guo2017time} respectively constructed a single index coefficient model and a partial linear model to deal with the non-linearity. 
In \cite{li2009single}, the authors proposed a single index additive autoregressive model for a multi-variate time series.
All of the above literature estimates the single index model in the time series data by penalized splines, which involved the selection for smoothing parameters. Large sample results were derived based on mixing conditions, yet non-asymptotic results are not provided, making them not applicable for high dimensional (large $M$) settings.

A multitude of advances on isotonic regression analysis (\cite{durot2002sharp, zhang2002risk, chatterjee2014new,chatterjee2015risk,bellec2018sharp}) substitutes the smoothness assumption by montonicity of the link function, which leads to the isotonic single index model. To estimate this semi-parametric model,
the \textit{Isotron} algorithm and estimators were proposed and studied in \cite{kakade2011efficient, kalai2009isotron}. To further address the high-dimensional challenges, a variable selection procedure using LASSO was combined with the isotonic single index model in \cite{neykov2019, foster2013variable}, both of which considered independent Gaussian data. In \cite{balabdaoui2019least}, the authors showed that the rate of the least squared estimator of the bundled isotonic single index function in the $\ell_2$ norm with respect to the sample size $n$ is $n^{-1/3}$ under appropriate conditions. In \cite{dai2021convergence}, the ``Sparse Orthogonal Descent Single-Index Model'' (SOD-SIM) is developed with the isotonic regression and a projection-based iterative approach, where a $n^{-1/3}$ convergence guarantee in the high dimensional setting is given. Both of these papers (\cite{balabdaoui2019least,dai2021convergence}) are focusing on the regression setting (\ie $\bbE(Y|X) = f^*(X^T u^*)$) with independent data.

Perhaps the most closely related prior work to our setting is \cite{wang2016isotonic}, which proposed an Isotonic-Hawkes process whose intensity function was formulated in the form of an isotonic single index model. They used an alternating minimization procedure in the algorithm, which shares the same framework as ours and showed the efficiency of the estimated near optimal indices and link functions. On the other hand, there are significant differences between this work and ours. Our results can not only be applied to the counting process as the Isotonic Hawkes process dealt with, but can also be applied to more general multi-variate time series with continuous-value co-variates. Besides, we take the influence network's sparsity into account and introduce a hard thresholding operator to enforce the sparsity, which is particularly helpful in the high dimensional setting. 


\subsection{Contributions}
Our major contributions in this paper are as follows:
\begin{itemize}
    \item We formulate the monotone single index multi-variate autoregressive model (monotone SIMAM) in the high-dimensional settings to learn the influence network from a multi-variate time series or point process data, which is more flexible and robust compared to existing parametric models, while entails more interpretability than other non-parametric ones. Based on this model, a feasible algorithm in an alternating framework combining the iterative hard-thresholding (IHT) method and suitable initialization is provided to solve the non-convex problem.
    
    \item In terms of theoretical analysis, we provide the convergence rate for our network estimator from the proposed algorithm in a non-asymptotic manner that applies to the high-dimensional setting using martingale concentration inequalities. The result indicates that after applying sufficiently many iterations, given a multi-variate time series with $T$ observations, the Frobenius norm of the influence network estimation error converges in the order of $O(T^{-\frac{1}{3}})$ up to some poly-log terms. Specifically, our rate depends on the sparsity of the network, the noise level of the data, the Lipschitz continuity of the monotone function, and the dimension of the network. In addition, the empirical one-step prediction error also has the rate of $O(T^{-\frac{1}{3}})$. 
    We also prove that the angle between our initialization and the true parameters is acute with high probability which is sufficient to guarantee our $O(T^{-\frac{1}{3}})$ rate.
     
    \item Simulation results are given to support the $O(T^{-\frac{1}{3}})$ convergence rate of the estimator derived from our algorithm with the nonlinear link functions unknown, after sufficiently many iterations. It is also illustrated that our method has a better performance in terms of both in-sample and out-of-sample prediction errors, compared to VAR with $\ell_1$ penalty.
    
    \item Two real data examples, the Chicago crime data and Memetracker data, are analyzed, which indicate that there exist highly nonlinear and non-smooth structures in point process data. In terms of prediction and estimation, we observe a significant advantage of our proposed monotone SIMAM over other popular parametric network estimating models, such as vector autoregressive(VAR) model, VAR with LASSO type of penalty and Poisson autoregressive(PAR) with the $\ell_1$ penalty.
\end{itemize}

\section{Preliminaries}

\subsection{Notations}

Let $\{X_0, X_2,...,X_T\} \subset \mathbb{R}^M$ be a $M$-dimensional time series with $T+1$ observed time points. Combining them together as rows of a matrix gives $\bfx \in \mathbb{R}^{(T+1)\times M}$, whose entries are $X_{i,j}\in \mathbb{R}$, where $i \in [T+1]-1 = \{0,\dots, T\}, j \in [M] =\{ 1,\dots, M\}$.

We write $\X_{-i} \in \mathbb{R}^{T\times M}$ as the matrix with $i$-th row deleted from $\X$; $\X_{-i,j} \in \mathbb{R}^{T}$, in this manner, denotes the $j^{th}$ column of the matrix $\X_{-i}$. Specifically for example, as will appear repeatedly in the paper, $\X_{-T}$ denotes all the data collected in time points $t = 0,1,\dots, T-1$; $\X_{-0,j}$ denotes the $j$-th co-variate observations in time points $t = 1,\dots, T$, with the first observation deleted.

For an index set $\I = \{i_1,\cdots,i_k\} \subseteq \mathbb N$ and any vector $\bfv = (v_{1},\dots,v_n)^T \in \mathbb{R}^n$ with length $n \geq \max\{\I\}$, we denote $\bfv_{\I} = (v_{i_1},\dots,v_{i_k})^T$ as the extraction of all elements in $\bfv$ whose indices are included in $\I$; $|\I| = k$ represents the cardinality of set $\I$. For any $l \in \mathbb{N}$, $\I + l$ is short for the set with all elements being added $l$: $\I +l = \{i_1+l,\dots, i_k+l \}$.

Let $\Phi_s: \bbR^M \rightarrow \bbR^M$ be the hard thresholding operator whose image is always a subset of all $s$-sparse vectors in $\bbR^M$, i.e.
\begin{equation}
      \Phi_s(x) = \argmin_{y\in \bbR^M} \{ \|y - x\|_2 : \| y\|_0 = s \}, ~\forall x\in \bbR^M.
\end{equation}
    
Another important projection operator we would use is the orthogonal projection operator $\mathcal{P}_u^{\perp}(\cdot)$ for any $u\in \bbR^M$, which projects any vector in $\bbR^M$ onto the subspace of $\bbR^M$ orthogonal to $u$:
\begin{equation}
    \mathcal{P}_u^{\perp}(x) = \argmin_{y\in \bbR^M}\{\| y - x\|_2^2: \langle u, y \rangle = 0\}, ~\forall x\in \bbR^M.
\end{equation}

\subsection{Partial ordering and Isotonic Regression}

The isotonic regression problem is: 
\begin{equation}
    \begin{split}
        &\text{Minimize } \sum_{i=1}^T \left(
            v_i - x_i
        \right)^2\\
        & \text{Subject to } x_i \leq x_j~\text{when } 
        i \preceq j,
    \end{split}
\end{equation}
where $\preceq$ is a specified partial ordering on $\Omega = \{1,2, \dots, T\}$. The solution to this problem is referred to as \emph{isotonic regression}. The vector $\mathbf{x} =(x_1,\dots, x_T)$ is said to be \emph{isotonic} or \emph{order preserving} if $i \preceq j$ implies $x_j \leq x_j$. Note that the set of isotonic vectors $\mathbf{x}\in \bbR^T$ is a closed convex cone, which guarantees that there is a unique solution $\mathbf{x}\in \bbR^T$ that solves the above isotonic regression problem.

To specify the partial ordering $\preceq$, we could introduce a collection of real numbers $\z = (z_1,\dots, z_T)$:
\begin{equation}
    i \preceq j \text{ if } z_i \leq z_j, ~\forall i,j \in \Omega.
    \label{eq: p_ordering}
\end{equation}

Hence, for any collections $\z = (z_1,\dots, z_T) $ and $\bfv = (v_1,\dots,v_T)$ , define $\iso_{\z}(\bfv)\in \bbR^T$ as the isotonic vector of $\bfv$ with respect to $\z$, i.e. the solution for the following constrained minimization:
\begin{equation}
    \iso_{\z}(\bfv) = \argmin_{\bfx \in \bbR^T}\{ \parallel \bfv - \bfx \parallel_2^2: x_i \leq x_j \text{ whenever } z_i \leq z_j \text{ for }\forall i,j \in [T]\}.
    \label{eq: iso}
\end{equation}

Essentially, $\iso_{\z}(\bfv)$ preserves the ordering of $\z$ while fitting $\bfv$. Important properties of isotonic regression include its contractiveness with respect to some seminorms, as is given in \cref{lemma: contractive} and \cref{cor: contract}. Computationally, using \emph{pool-adjacent-violators algorithm (PAVA)} \cite{mair2009isotone} with $\z,\bfv \in \bbR^T$ being the algorithm inputs, $\iso_{\z}(\bfv)$ can be solved with a computational complexity $O(T)$.

Since the above partial ordering in \cref{eq: p_ordering} is only defined by collections of scalars, when it turns to high dimensional vectors, such ordering has to be induced by some projection. Consider a collection of $T$ points $X_1,\dots, X_T$ in $M$-dimensional Euclidean space and a reference vector $\bfu\in \bbR^M$, for a permutation $\pi$ of the set $\{1,\dots, T\}$, if 
\begin{equation}
    \langle X_{\pi(1)}, \bfu\rangle \leq \dots \leq \langle X_{\pi(T)}, \bfu\rangle,
\end{equation}
we say that $\bfu \in \bbR^M$ \emph{induces} the ordering $\pi$. 

\section{Model and Algorithm}
\subsection{Monotone Single Index Multivariate Autoregressive Model}

We assume that the time series $\{X_t\}_{t = 0}^T \subset \bbR^M$ follows the monotone single index multi-variate autoregressive model (SIMAM) and is conditionally independent across $j\in [M]$, \ie
\begin{equation}\label{eq: sim model}
    \bbE(X_{t,j}|X_{t-1}) = f_j^{*}(X_{t-1}^T u_j^{*})
\end{equation}
for all $t\in [T]$ almost surely with an unknown index $u_j^* \in \bbR^M\symbol{92} \{\mathbf{0}\}$ and a monotone function $f_j^*$, which is also unknown. According to the conditional independence across $j\in [M]$, conditioned on the previous data, the elements $X_{t,1},\dots, X_{t,M}$ of the $t$-th observation are independent of one another.
Specifically, the index $u_j^* = (u_{j1}^*, \dots, u_{jM}^*)^T$, also known as the direction vector, is assumed to lie on a unit sphere $\cS^{M-1} \subset \bbR^M$ with $s_j^*$ nonzero elements, where $s_j^*$ refers to the sparsity parameter. Let $$s^* := \max \{s_1^*, \dots, s_M^*\},$$ 
then $s^*$ is the maximum in-degree of the directed graph induced by the network, in which the $i$-th node represents the $i$-th co-variate, and the edge from node $i$ to node $j$ exists when $u^*_{ji}\neq 0$ .

The univariate function $f_j^*: \mathbb{R} \rightarrow \mathbb{R}$, which contributes to the non-parametric flexibility of the model, is assumed to be $L_j$-Lipschitz continuous and non-decreasing on its domain that contains the range of the linear predictors $\{X_{t-1}^T u_j^*\}_{t=1}^T$. 
Let $\cM$ denote the function class that contains all monotonically non-decreasing functions, then we have $f_j^* \in \cM$ for any $j\in [M]$. For technical reasons, we extend all functions outside their actual support by taking the extension to be constant to the left and right of the original support's endpoints.

The noise terms in the $j$-th co-variate, denoted as $\Z_j = (Z_{1,j}, \dots, Z_{t,j})$ where
\begin{equation}
   Z_{t,j}= X_{t,j} - f_j^*(X_{t-1}^T u_j^*),
\end{equation}
are generally assumed to be martingale differences with conditional sub-Gaussian tails. 

Due to the monotonicity of $f_j^*$ for each $j\in [M]$, the nonlinear function $f_j^*$ is \emph{order-preserving}, while the index $u_j^*$ essentially captures the direction that finds the best ordering for the variables to project on, \ie $u_j^*$ is the linear projector that \emph{induces} the variable ordering. 

Based on this model, we can make inference on the influence network by the direction vectors $\{u_j^*\}_{j=1}^M$.
Let $A^* = (u_1^*, \dots, u_M^*)$ denote the coefficient matrix with the direction vectors being its columns, then its element $A_{ij}$ with $i,j \in [M]$ represents the temporal influence of $X_{t-1,i}$ on $X_{t,j}$ for every time series observation $t\in [T]$. From the perspective of graphs/networks, $A^*$ is an adjacency matrix of the weighted \emph{directed} graph that indicates the influence network.

\subsection{Connection with Generalized Linear autoregressive Models}

In the generalized linear autoregressive models (GLAR)~\cite{hall2016inference}, for any $j\in[M]$, the density of $X_{t+1,j}$ given $X_{t} = x_t$ with respect to a given base measure is an exponential family of the form 
\begin{equation}
  p(x_{t+1,j}| X_t = x_t) =  h_j(x_{t+1,j},\phi_j) \exp\left\{
    \frac{x_{t+1,j}(x_t^Tu_j^* +\nu_j)  - Z_j(x_t^Tu_j^*+\nu_j)}{\phi_j}
    \right\},
    \label{eq: ef}
\end{equation}
where  $h_j$ is the base measure of the conditional distribution, and $\phi_j>0$ is the dispersion parameter. $u_j^*$ is the unknown vector containing network parameters of our interest, while $Z_j(\cdot)$ is referred to as the log partition function, whose second-order derivative satisfies $Z_j''(\cdot) >0$ for all elements in its domain. Further since $p(\cdot)$ belongs to an exponential family, we have
\begin{equation}
  \bbE(X_{t+1,j}| X_t)  = Z_j'(X_t^Tu_j^*+\nu_j).
\end{equation}

By the convexity of the log partition function $Z_j(\cdot)$, we know that its first-order derivative $Z_j'(\cdot)$ is monotonically non-decreasing. Thus there exists a monotone function $f_j^*$ for each $j\in[M]$, such that the monotone single index multi-variate autoregressive model \cref{eq: sim model} holds true.

Hence, the generalized linear autoregressive model is a special case of the monotone SIMAM. A key difference between these two models is that in GLAR, the inverse link function (or transfer function) $f_j^*(\cdot) = Z_j'(\cdot) + c$ , where $c$ is a constant thanks to the monotonicity of $Z_j'$, is assumed known. In the monotone SIMAM, however, the function $f_j^*$ is allowed to be unknown and only assumed to be isotonic. Further, the conditional distribution of $X_{t+1,j}$ given $X_t$ is no longer assumed to take the form \cref{eq: ef}, making the model more flexible.

\subsection{Non-convex optimization and algorithm}

Based on the monotone SIMAM, we want to find the solutions to the following non-convex optimization problem over the unknown functions $\{f_1,\dots, f_M\}$ and direction vectors $\{u_1,\dots, u_M\}$:
given a multi-variate time series or point process $\{X_0, \dots, X_T\}$ and sparsity levels $\{s_1, \dots, s_M\}$, 
    \begin{equation}
        \begin{split}
          &\text{Minimize } \frac{1}{T} \sum_{j=1}^M\sum_{t=0}^{T-1} \left[ 
        X_{t+1,j} - f_j (X_t^T u_j)
    \right]^2\\
    &\text{Subject to } \|u_j\|_2 = 1;~ \|u_j\|_0 = s_j \text{ and } f_j \text{ is non-decreasing, } ~\forall j\in [M].
        \end{split}
    \end{equation}

Since $\{X_{t+1,1},\dots, X_{t+1,M}\}$ is conditionally independent given the prior observation $X_t$, the loss function is separable with respect to the sum over $j\in [M]$. Therefore, we can estimate the pairs $\{(f_j^*, u_j^*)\}_{j=1}^M$ separately. For any $j\in[M]$, to simultaneously estimate $f_j^*$ and $u_j^*$, an alternating procedure is proposed. Note that once the direction ${u}_j \in \bbR^M$ is given, $f_j$ can be estimated by minimizing the profile loss function $\cL_{j} (f;  u_j)$ where
\begin{equation}
  \cL_{j}(f; u) = \frac{1}{T}\sum_{t = 0}^{T-1}
  \left[
    X_{t+1,j} - f(X_{t}^T u)
  \right]^2.
\end{equation}
By the definition of isotonic regression in \cref{eq: iso}, we know that 
\begin{equation}
  \iso_{\X_{-T} u_j} (\X_{-0,j}) = (\hat{f}_j(X_0^Tu_j), \dots, \hat{f}_j(X_{T-1}^Tu_j)) \text{ for any } \hat{f}_j \in  \arg\min_{f\in \cM} \cL_j(f;  u_j),
  \label{eq: est_f}
\end{equation}
 where $\X_{-T} u_j = (X_0^T u_j, \dots, X_{T-1}^T  u_j)^T $ and $\X_{-0,j} = (X_{1,j},\dots, X_{T,j})^T$.

On the other hand, even if the non-decreasing function $f_j$ is given, minimizing $\cL_j(f_j;u)$ over $\{u\in \bbR^M: \|u\|_2=1 \}$ with the non-convex constraint $\|u\|_0 = s_j$ is still challenging. In linear settings without the non-linear transformation $f_j(\cdot)$, \emph{projected gradient descent} (PGD) (also known as \emph{iterative hard-thresholding} (IHT)) algorithms are used to solve the $\ell_0$-norm constrained problem (see \eg \cite{blumensath2009iterative,jain2014iterative}). 
To implement PGD, we first need to find the gradient with respect to $u$ in $\cL_j(f_j;u)$: (for heuristic purpose, we assume $f_j$ has the first-order derivative $f_j'$ here)
\begin{equation}
  \nabla_{u} \cL_j(f_j; u) = \frac{1}{T} \sum_{t=0}^{T-1} \left[ X_{t+1,j} - f_j(X_t^Tu)\right]\cdot f_j'(X_t^T u)\cdot X_t.
\end{equation}
Although $f_j'(\cdot)$ is unknown, due to the monotonicity we know that $f_j'(X_t^T u)$ is a non-negative scalar for all $t\in[T]-1$, thus in the gradient descent step, we remove this term as an approximation to the gradient direction. Therefore for any $u_j$, if we have obtained an estimated $\hat{f}_j$ satisfying \cref{eq: est_f}, to update $u_j$ using PGD, the pseudo gradient we use is 
\begin{equation}
  \frac{1}{T} \sum_{t=0}^{T-1} \left[ X_{t+1,j} - \hat{f}_j(X_t^T u_j)\right]\cdot X_t = 
  \frac{1}{T}\X_{-T}^T\left[\X_{-0,j} - 
  \iso_{\X_{-T} u_j} 
  (\X_{-0,j})\right],
\end{equation}
the specific usage of which would be further indicated in \cref{eq: pgd}.

With the above alternating framework, we are now in position to introduce the procedure to estimate the direction vector $u_j^*$ and the corresponding monotone link function $f_j^*$. The pseudo-code of the overall procedure is given in \cref{alg: SIMAM}.

For any $j \in [M]$, fix a step-size $\eta_j = \frac{1}{L_j\beta}$ and a maximum iteration count $K_j$, where $\beta$ is the largest eigenvalue of $\X_{-T}$ defined in \cref{asspt: data_bound}. We use $s_j$ as an estimated sparsity level for the $j$-th direction vector. For theoretical convenience, we require $s_j$ to be larger than the true sparsity $s_j^*$ which is often standard.
The algorithm is:
\begin{enumerate}
  \item Initialization:
  \begin{equation}
    \tilde{u}_j^{(0)} =\frac{1}{T}\X_{-T}^T(\X_{-0,j} - \overline{\X_{-0,j}}\cdot \mathbf{1}_T) ;
    \label{alg: init1}
  \end{equation}
  where $\overline{\X_{-0,j}}$ denotes the mean of $\X_{-0,j}\in \bbR^T$, and $\mathbf{1}_T$ is the all-one vector in $\bbR^T$. Taking hard-thresholding and normalization to enforce the sparsity and unit norm, we have
  \begin{equation}
    u_j^{(0)} = \frac
    {\Phi_{s_j}\left(
    \tilde{u}_j^{(0)}
    \right)}
    {\parallel \Phi_{s_j}\left(
      \tilde{u}_j^{(0)}
    \right) \parallel_2}.
    \label{alg: init2}
  \end{equation}
  \item {
    In each iteration $k = 1,\cdots, K_j$, 
    \begin{enumerate}
      \item Compute $\iso_{\X_{-T}u_j^{(k-1)}} (\X_{-0,j})$;
      \item Take an orthogonal pseudo gradient step,
      \begin{equation}
        \tilde{u}_j^{(k)} = u_{j}^{(k-1)} + \eta_j \cdot \mathcal{P}_{u_j^{(k-1)}}^{\perp} \left(\frac{1}{T}\X_{-T}^T\left[\X_{-0,j} - 
        \iso_{\X_{-T} u_j^{(k-1)}} 
        (\X_{-0,j})\right]\right);
      \label{eq: pgd}
      \end{equation}
      \item {Enforce sparsity and unit norm,
    \begin{equation}
      u_j^{(k)} = \frac{\Phi_{s_j}(\tilde{u}_j^{(k)})}{\parallel\Phi_{s_j}(\tilde{u}_j^{(k)})\parallel_2};
    \end{equation}
  }
  \item Stop when $k = K_j$.
    \end{enumerate}
  }
\end{enumerate}

Note that for any given $u \in \bbR^M$, the minimum of $f\rightarrow \cL_{j}(f;u)$ over the monotone function class $\mathcal{M}$ can always be achieved, yet the minimizer is not unique over $\cM$. 
In fact, it is only uniquely defined at the points $\{X_{t}^T u\}_{t=0}^{T-1}$ (see \emph{theorem 2.1} in \citep{balabdaoui2019least}).
In other words, the best isotonic function that regresses $\X_{-0,j}$ on $\X_{-T}u_j^{(k-1)}$ is uniquely determined only at all the observed linear predictors $\X_{-T}u_j^{(k-1)}$, which take the value $\iso_{\X_{-T}u_j^{(k-1)}} (\X_{-0,j})$.

Such uniqueness makes the above estimation procedure work well, yet when it comes to prediction, new data outside of the support of previous observations may not give well-defined predictions. 
Thus we consider below the estimated monotone function $f_j^{(k)}$ to be left continuous and piece-wise constant, with jumps only possible at $T$ linear predictors $\langle X_0, u_j^{(k-1)}\rangle, \dots, \langle X_{T-1}, u_j^{(k-1)} \rangle$, for the sake of convenience. 
In practice, we would use the classical algorithm PAVA (Pool Adjacent Violators Algorithm) \citep{mair2009isotone} to compute this least squares estimator $\iso_{\X_{-T}u_j^{(k)}}(\X_{-0,j})$ in each iteration step.

\begin{algorithm}[ht!]
    \caption{Alternating Projected Gradient Descent for monotone SIMAM}
    \label{alg: SIMAM}
    \textbf{Input}: $M$-dimensional time series $\{X_0, X_1, \cdots, X_T\} \subset \mathbb{R}^M$ \;
    \textbf{Parameters}: sparsity level $s_j$, step size $\eta_j$, iteration number $K_j$, for any $j \in [M]$\;
    \ForEach{
      $j = 1,\cdots,M,$
    }{
      Initialize with $u_j^{(0)} = \frac{\Phi_{s_j}(\tilde{u}_j^{(0)})}{\| \Phi_{s_j}(\tilde{u}_j^{(0)})\|_2}$ where $\tilde{u}_j^{(0)} =\frac{1}{T}\X_{-T}^T(\X_{-0,j} - \overline{\X_{-0,j}}\cdot \mathbf{1}_T)$\;
      
      \ForEach{
        iteration $k = 1,\cdots,K_j$,
      }{
        Compute $\iso_{X_{-T} u_j^{(k-1)}}(\X_{-0,j})$ with PAVA\;
        $\tilde{u}_j^{(k)} = u_{j}^{(k-1)} + \eta_j \cdot \mathcal{P}_{u_j^{(k-1)}}^{\perp} (\frac{1}{T}\X_{-T}^T(\X_{-0,j} - 
        \iso_{\X_{-T} u_j^{(k-1)}} 
        (\X_{-0,j}))$
        \;
        ${u}_j^{(k)} = \frac{\Phi_{s_j}(\tilde{u}_j^{(k)})}{\parallel \Phi_{s_j}(\tilde{u}_j^{(k)}) \parallel_2}$ \;
      }
    }
    \textbf{Output}: ${u}_1^{(K_1)},\cdots, {u}_M^{(K_M)}.$
\end{algorithm}

\section{Main results}

\subsection{Assumptions}

\subsubsection{Assumptions for model identifiability}

To ensure identifiablity, we assume $u_j^{*}$ lies on the unit sphere, i.e. $\parallel u_j^{*}\parallel_2 = 1$; 
$f_j^{*}$ is a monotonically non-decreasing function with $L_j$-Lipschitz continuity: for any $x\in \bbR$ and $\Delta_x>0$,
\begin{equation}
    0 \leq f_j^{*}(x+\Delta_x) - f_j^{*}(x) \leq L_j\cdot \Delta_x.
    \label{asspt: mono_Lip}
\end{equation}

Such non-parametric function class, due to its large complexity, still suffers from identifiability issues, thus we use the following condition to ensure that the model is identifiable: 
for $j = 1,\dots, M$,
there exists a small relaxation term $\epsilon_j \geq 0$ and a positive value $\alpha_j> 0$, such that
\begin{equation}
    \frac{1}{N}\parallel f_j(\X_{-N}\cdot u_j) - f_j^{*}(\X_{-N}\cdot u_j^{*})  \parallel_2^2 \geq \alpha_j\parallel u_j - u_j^{*}\parallel_2^2 - \epsilon_j^2, 
    \label{eq: identifiability}
\end{equation}
for any monotonically non-decreasing $f_j$ and $s_j$-sparse unit vector $u_j$. Note that in this paper, for any univariate function $f$ and any vector $\bfv$, $f(\bfv)$ denotes the vector of the same length as $\bfv$ with $[f(\bfv)]_i = f(\bfv_i)$.

\begin{remark}
    Suppose we further assume the function class containing $f_j$ and $f_j^*$ of interest could be uniformly lower bounded by a linear function. In that case, we know that assumption \cref{eq: identifiability} is equivalent to the restricted eigenvalue condition (REC). Indeed, this assumption implicitly involves conditions analogous to restricted strong convexity to capture the dependence structure. In much of the literature, this condition can be verified in cases of independent design. While the dependence structure introduced in our autoregressive framework makes it a more complex condition to verify, such REC type of assumption is also included in \cite{Mark_2019} when dealing with dependent data.

    In fact, the identifiability of the class of non-decreasing functions $\mathcal{M}_j$ in which $f_j$ lies in, depends on the data structure of $\X_{-T}$ and properties of the true monotone function $f_j^*$. For example, in order to include all constant functions in the identifiable function class, we need 
    \begin{equation}
        \inf_{c\in \bbR} \frac{1}{T} \| f_j^*(\X_{-T} \cdot u_j^*) -c \mathbf{1}_T \|_2^2 = \| f_j^*(\X_{-T} \cdot u_j^*) -\overline{f_j^*(\X_{-T} \cdot u_j^*)} \mathbf{1}_T \|_2^2 \geq 4 \alpha_j - \epsilon_j^2,
    \end{equation}
    \ie the underlying signals $\left\{f_j^*(\langle\X_0, u_j^*\rangle ), f_j^*(\langle\X_1, u_j^*\rangle ),\dots, f_j^*(\langle\X_{T-1}, u_j^*\rangle )\right\}$ should have a variance no less than $4\alpha_j - \epsilon_j^2$. 
\end{remark}

\subsubsection{Noise distribution assumptions}

For any $j\in [M]$, the noise sequence
$\{Z_{t,j}\}_{t=1}^T$ are assumed to be a martingale difference sequence satisfying the $\sigma_j$-sub-Gaussian tail condition:
    \begin{equation}
        \bbE\left[ | Z_{t,j} | \right] < \infty;~\bbE\left[Z_{t,j} | \mathcal{F}_{t-1}\right] = 0;~ \bbE\left[e^{\lambda Z_{t,j}} | \mathcal{F}_{t-1}\right] \leq e^{\lambda^2 \sigma_j^2 /2}, ~ \forall t = 1,\dots,T.
        \label{asspt: noise}
    \end{equation}

\begin{remark}
    Being a martingale difference sequence, the noise sequence $\{Z_{t,j}\}$ are allowed to be signal-dependent. Combined with the sub-Gaussian tail condition, such an assumption is weak enough to include many popular distribution assumptions for a time series. For example, $\{Z_{t,j}\}$ could be independent mean-zero Gaussian noise with variance $\sigma_j^2$, which is a common assumption in analyzing continuous data. Another popular situation is to deal with the count data, when $Z_{t,j}, t = 1,\dots, T$ are typically signal-dependent noises from the Poisson Auto-Regressive (PAR) model:
    \begin{equation}
        X_{t,j}|\mathcal{F}_{t-1} \sim \text{Poisson}(f_j^*(X_t^T u_j^*)).
    \end{equation}
    In this case, $\{Z_{t,j}\}_{t=1}^T$ is still a martingale difference sequence. Although the Poisson tail is heavier than the sub-Gaussian tail, we can always conduct a truncation on the tails to make it sub-Gaussian. In practice, as long as the noise is bounded by a constant
    $\sigma_j$, \ie $|Z_{t,j}| \leq \sigma_j$, the $\sigma_j$  sub-Gaussian condition can be met.
\end{remark}

\subsubsection{Assumptions for data boundedness}
We assume the observed data $\X_{-T} = (X_0^T,\dots, X_{T-1}^T)^T \in \bbR^{T \times M}$ is entry-wise bounded:
\begin{equation}
    \max \{|X_{t,j}|: t \in \{[T]-1\}  \text{ and } j \in [M]\} \leq M_x <\infty.
    \label{asspt: entry_bound}
\end{equation}
The eigenvalues of $\X_{-T}$ are also upper bounded in the sparse setting:
\begin{equation}
    \frac{1}{T} \parallel \X_{-T}\cdot u\parallel_2^2 \leq \beta \parallel u \parallel_2^2, ~\forall u \in \mathcal{S}^{p-1} \text{with sparsity at most } \max_{j\in [M]}(2s_j+s_j^{*}). 
    \label{asspt: data_bound}
\end{equation}

\subsection{Convergence Guarantee}

The following main theorem (\cref{thm: coefficient conv}) gives the non-asymptotic result for our network estimator from \cref{alg: SIMAM}, which indicates that our algorithm \cref{alg: SIMAM} converges at a geometric rate, and after sufficiently many iterations, it converges to the statistical error with the rate $O(T^{-\frac{1}{3}}\sqrt{s\log(TM)})$ up to poly-log terms.

\begin{theorem}
    Suppose the $M$-dimensional multi-variate time series data follows the monotone SIMAM in \cref{eq: sim model}, and the assumptions for Lipschitz continuity (\cref{asspt: mono_Lip}), model identifiability (\cref{eq: identifiability}), sub-Gaussian martingale noise(\cref{asspt: noise}) and boundedness (\cref{asspt: entry_bound}, \cref{asspt: data_bound}) are satisfied. Denote $\Delta_j$ as the quantity listed in \cref{eq: def_Delta}. For any $j \in [M]$, after running \cref{alg: SIMAM} for $K$ times, with the step size being $\eta_j = \frac{1}{L_j\beta}$ and hard-thresholding sparsity level $s_j$ satisfying
    \begin{equation}
        s_j^* < s_j \cdot \min\left\{ 
            1 - \frac{\delta_j \alpha_j}{L_j^2\beta}
        ,~ (1-\delta_j)\frac{\alpha_j}{L_j^2\beta} - \frac{\epsilon_j^2}{2L_j^2\beta} - \frac{\Delta_j^2}{2\delta_j\alpha_j \beta} \right\}^2
        \label{asspt: sparsity}
    \end{equation}
    for some $0<\delta_j<1$, the following bound holds with probability at least $1- 4\gamma^{2s_j}$:
    \begin{equation}
        \begin{split}
            \parallel u_j^{(K)} - u_j^{*} \parallel_2^2 \leq 2\theta_j^K + R_j^2,\text{ where } &\theta_j = \frac{1 - \frac{\alpha_j}{L_j^2 \beta}}{1 - \frac{\delta_j\alpha_j}{L_j^2\beta} - \sqrt{\frac{s_j^{*}}{s_j}}} <1, \\
           \text{and}~ &R_j^2 = \frac{\frac{\epsilon_j^2} {L_j^2\beta} +\frac{\Delta_j^2}{\delta_j\alpha_j\beta}}{(1-\delta_j)\frac{\alpha_j}{L_j^2\beta} - \sqrt{\frac{s_j^*}{s_j}}} = O\left(T^{-\frac{2}{3}} s_j \log (\frac{TM}{\gamma}) \right),
        \end{split}
    \end{equation}
    as long as we have a warm initialization satisfying $\langle u_j^{(0)}, u_j^{*}\rangle \geq 0$.
    \label{thm: coefficient conv}
\end{theorem}

Therefore, for any tolerance $\tau_j > R_j$, running \cref{alg: SIMAM} for $K_j \geq \frac{\log(\frac{\tau_j^2 - R_j^2}{2})}{\log \theta_j}$ many iterations will guarantee that $\| u_j^{(K_j)} - u_j^{*}\|_2 \leq \tau_j$. 
With the conditions in \cref{thm: coefficient conv} satisfied, let $s = \max_{j=1,\dots, M} s_j$ be the thresholded maximum in-degree of the network, after running \cref{alg: SIMAM} for sufficiently many iterations in the $j$-th co-variate for $j\in [M]$, the influence network estimator $\widehat A = (u_1^{(K_1)}, \dots, u_M^{(K_M)})$ approaches the truth $A^*$ with the rate
\begin{equation}
    \| \widehat A - A^{*}\|_F \cong O_p\left( T^{-\frac{1}{3}} \sqrt{s \log(TM)}\right).
\end{equation}

\begin{remark}

    The remaining term $R_j$ gives the statistical error bound in \cref{thm: coefficient conv}. To assure the convergence, the sparsity level $s_j$ used in the hard-thresholding projection should be larger than its truth $s_j^*$, which is commonly needed in the projected gradient descent literature~(see \eg \cite{jain2014iterative}), due to the greedy nature of IHT algorithm. The trade-off is that a larger sparsity threshold $s_j$ speeds up the algorithm convergence, while sacrificing the statistical error rate in $R_j$, as is indicated in \cref{thm: coefficient conv}.
    
\end{remark}

In \cref{thm: coefficient conv}, a good starting point for the algorithm that has an acute angle with the truth is needed to ensure the convergence. The following lemma shows that our proposed initialization in \cref{alg: SIMAM} satisfies this condition with high probability.
    
\begin{lemma} (Initialization guarantee)
    The angle between the initialization $u_j^{(0)}$ in \cref{alg: SIMAM} and the truth $u_j^*$ is acute for any $j\in [M]$:
    \begin{equation}
        \langle u_j^{(0)} , u_j^* \rangle >0,~j\in [M],
        \label{eq: good_init}
    \end{equation}
    as long as $s_j > s_j^* \cdot \max\{1, \left(\frac{L_j}{2(4\alpha_j-\epsilon_j^2)}+O(\frac{\beta L_j^3}{T})\right)^2\}$, with high probability
    $ 1- 2(M+1) \exp(-\frac{T {U_j^{+}}^2}{8\sigma_j^2 M_x}),$
    where
    \begin{equation}
        U_j^{+} = \frac{1}{2}\left(
            \sqrt{\frac{1}{s_j} + 4\left[(\frac{1}{L_j\sqrt{s_j s_j^*}}-\frac{2\beta}{T}) 
            (4\alpha_j -\epsilon_j^2)- \frac{1}{2s_j}\right]}-\sqrt{\frac{1}{s_j}}
        \right).
    \end{equation}
    \label{lemma: good_init}
\end{lemma}

\begin{proposition}
    For each $j\in [M]$, after running \cref{alg: SIMAM} with a sufficient number of iterations such that the error $\| \hat{u}_j- u_j^*\|_2$ of the $j$-th variable estimate $\hat u_j$ is dominated by its statistical error term $R_j = O_p(T^{-1/3})$ up to some poly-log terms in \cref{thm: coefficient conv}, with high probability we have
    \begin{equation}
        \frac{1}{\sqrt{T}}\left\|
     \iso_{\X_{-T} \hat{u}_j} (X_{-0,j}) - f_j^*(\X_{-T} u_j^*) \right\|_2 \leq O_p\left(T^{-\frac{1}{3}}\log(T)\right).
    \end{equation}
    \label{prop: prediction}
\end{proposition}
Let $\hat{f}_j$ denote the monotone function estimated based on $\hat u_j$ (see \cref{eq: est_f}), then the result in \cref{prop: prediction} can also be expressed as:
\begin{equation}
  \sqrt{\frac{1}{T}\sum_{t=0}^{T-1} \left[\hat{f}_j(X^T_{t}\hat{u}_j) - f_j^*(X^T_{t} u_j^*) 
 \right]^2} \leq O_p\left(T^{-\frac{1}{3}}\log(T)\right),
\end{equation}
which demonstrates that \cref{prop: prediction} essentially gives the in-sample prediction error bound.

\section{Simulation Study}

In this section, we validate our theoretical results and explore the properties of monotone SIMAM on synthetic data. For a given number of time points $T$ and node size $M$, we construct a large sparse $M\times M$ coefficient matrix whose $j$-th column $u_j^*$ lies on a unit sphere and has $s_j^*$ non-zero elements for any $j\in [M]$. 
We define a sequence of monotone nonlinear functions $\{f_j^*\}_{j=1}^M$ in the form of
\begin{equation}
    f_j^*(x) = \frac{\exp(j\cdot x)}{\exp(j\cdot x)+1}, ~j\in[M].
    \label{eq: sim funcs}
\end{equation}
These functions are isotonic, Lipschitz continuous and well-bounded. 
When $j=1$, $f_j^*$ is the logistic link function.
As the value of $j$ increases, the nonlinearity of $f_j^*(x)$ increases accordingly.
First we consider the noise variables to be drawn from a Gaussian distribution. Using the equation $X_{t+1,j} = f_j^*(X_t^T \beta_j) + Z_{t+1,j}$, where $\{Z_{t,j}\}$ for ${j\in[M],t\in [T]}$ is the sequence of noise generated independently from N$(0, \sigma^2)$, we define a multi-variate time series $\{X_t\}_{t = 0}^T$ whose initial vector $X_0$ is randomly generated from a normal distribution N$(0,I_M)$.
Specifically, we take the dimension $M=9$, the sparsity $s_j^* = 3, \forall j \in [M]$ and noise level $\sigma = 0.05$. The sample size $T$, also referred to as the length of the time series, takes values from an integer sequence $\{50\times(i+1)\}_{i=1}^{20}$.

Given the above generated time series$\{X_0, X_1,\cdots, X_T\}$, we estimate the underlying network structure by using \cref{alg: SIMAM}. We take sparsity thresholds in \cref{alg: SIMAM} as $s_j = 4$ for $j= 1,\cdots,9$, which are slightly larger than the true sparsity levels ($s_j^*  = 3$). 
We set constant step-sizes $\eta_j$ for $j= 1,\cdots,9$ as $0.1$, and the maximum steps for iteration for all dimensions are set to be sufficiently large enough as $K_j = 2000, j = 1, \cdots,9$. Once the direction vectors $\{\hat{u}_1,\cdots, \hat{u}_9\}$ are estimated, we calculate the RMSE (square root of mean squared error) in the form of $\sqrt{\frac{1}{9}\sum_{j=1}^9\parallel \hat{u}_j -u_j^* \parallel_2^2}$, which is in fact the Frobenius norm of the influence network estimation error: $\| \widehat{A} - A^* \|_F$.

For each given sample size $T \in\{100, 150,\cdots, 1050\}$, we independently generate $100$ multi-variate time series of length $T$ with different random seeds and perform the above estimation procedure repeatedly for each of the $100$ realizations. At each sample size, we take the average over the $100$ network estimation errors. Plotting the network estimation against the sample size $T$ and in particular against $T^{-1/3}$, as is shown in \cref{fig: convergence}, we find that the estimation error converges in the rate $O(T^{-1/3})$ which supports the theoretical analysis.

\begin{figure}[ht!]
    \centering
    \includegraphics[width=0.98\linewidth]{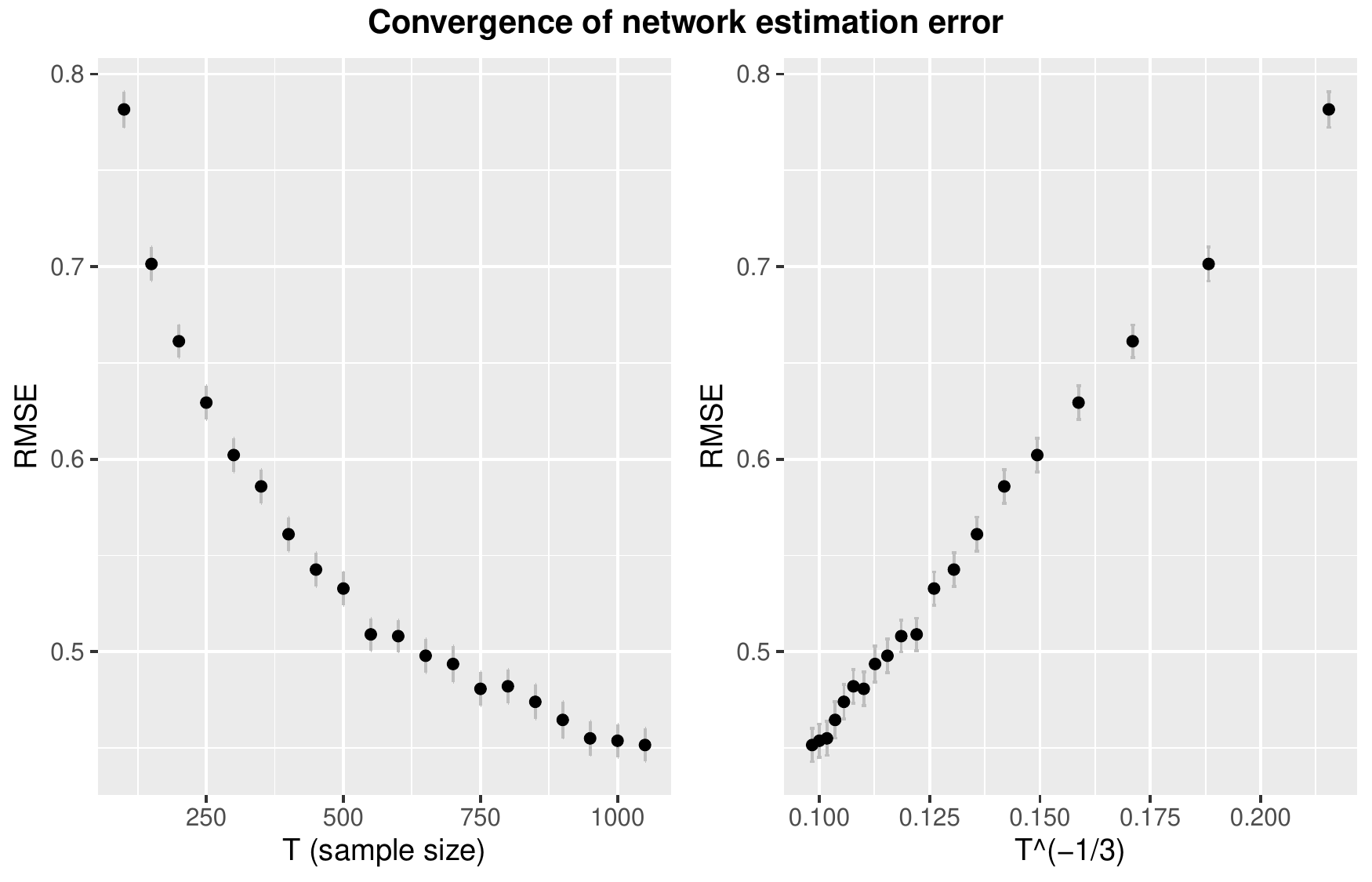}
    \caption{The convergence of the network (coefficient matrix) estimation in RMSE with respect to the sample size}
    \label{fig: convergence}
\end{figure}

Another thing we are interested in is whether the isotonic SIMAM framework works better when there is unknown non-linear structure in the data compared with other popular autoregressive methods in terms of prediction. 
We start with multivariate time series of node size $M = 9$
with monotone functions $f_2, \dots, f_{10}$ in \cref{eq: sim funcs}. (In the later example, we will increase the node size to $M = 36$ to see the prediction performance in a higher-dimension setting.)
The true coefficient matrix $A^*$ is randomly generated satisfying 
\begin{equation*}
    \|A^*_j\|_2 =1;~ \|A^*_j\|_0 = s_j^* = 3,~ \text{for any }j \in [9],
\end{equation*}
where $A_j^*$ is the $j$-th column of $A^*$.
Two types of noise distributions are considered: 
(a) Gaussian noise from N$(0,0.05^2)$; 
(b) Bounded Uniform noise from Uniform$(-0.1,0.1)$. 
The total number of observations is 1000, split into a training set (the first $9/10$ samples) and a testing set (the remaining $1/10$ samples).

As indicated in the main result (see \cref{thm: coefficient conv}), as long as the initialization satisfies $\langle u_j^{(0)}, u_j^* \rangle \geq 0$, the convergence of \cref{alg: SIMAM} is guaranteed. 
Therefore \cref{alg: SIMAM} can be adapted by replacing the initialization in \cref{alg: init1} with other initializations satisfying the above criteria. 
Here, we use the solutions from the LASSO method as a warm start for \cref{alg: SIMAM} to approximate SIMAM model. 
The step-size for \cref{alg: SIMAM} is taken as $0.01$, and the hard-threshold levels $s_j, j\in[M]$ are still taken as $4$, not accurate but slightly larger than the true parameters $s_j^* = 3$.
The prediction results in MSE are evaluated both on the training (in-sample) and test (out-of-sample) data. 
As is shown in \cref{fig: sim_9dim}, for both noise types of data, SIMAM performs better not only on the in-sample data, but also on the out-of-sample data with a better generalization performance. 
The convergence is achieved within $100$ steps, yet the generalization error increases slightly after the convergence in \cref{fig: uniform9dim}, suggesting that an early stopping after the convergence is recommended.

\begin{figure}[ht!]
    \begin{subfigure}[b]{.48\linewidth}
        \centering
        \includegraphics[scale = 0.7]{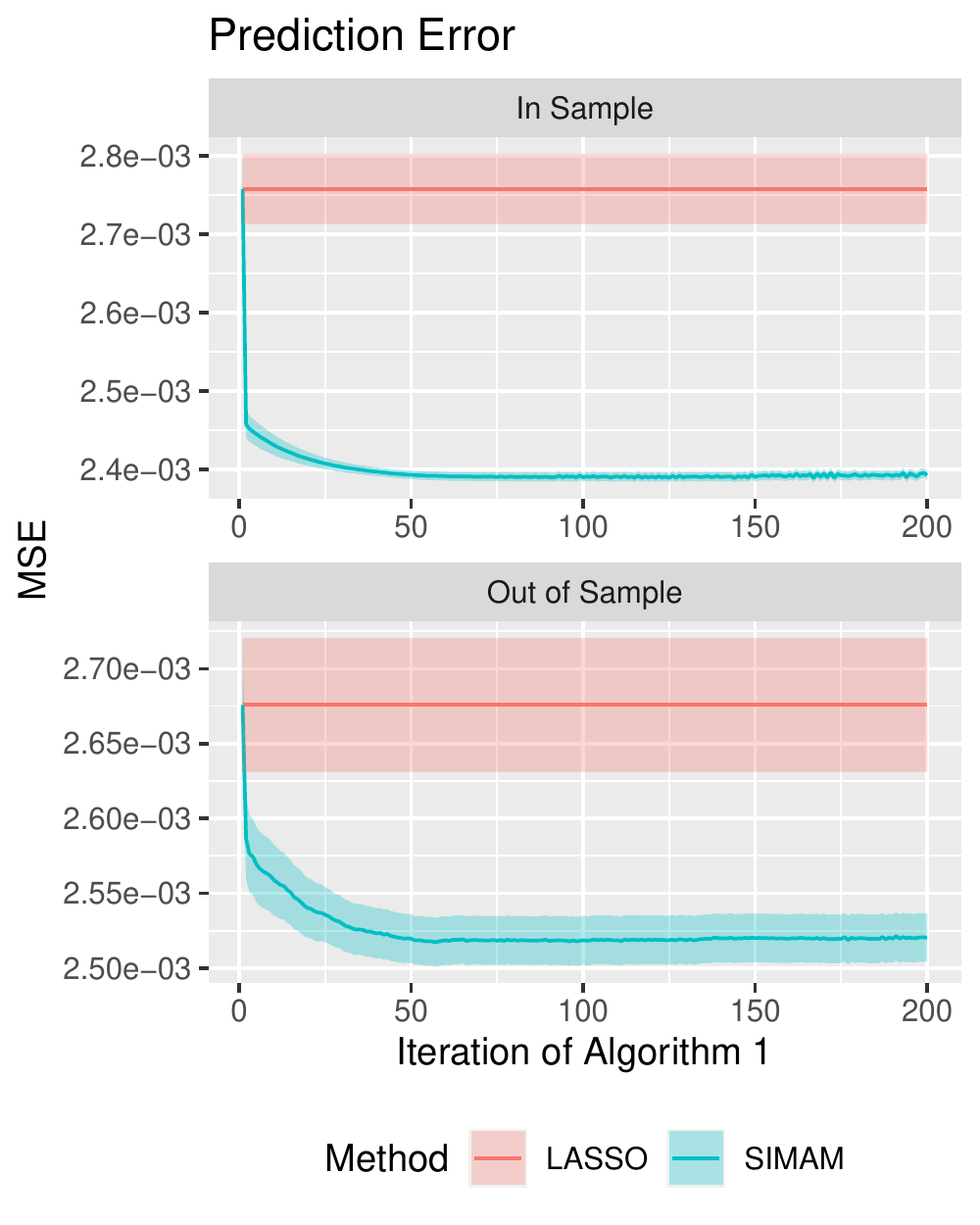}
        \caption{Time series with N$(0, 0.05^2)$ noise}\label{fig: gaussian9dim} 
    \end{subfigure}
    \begin{subfigure}[b]{.48\linewidth}
        \centering
        \includegraphics[scale = 0.7]{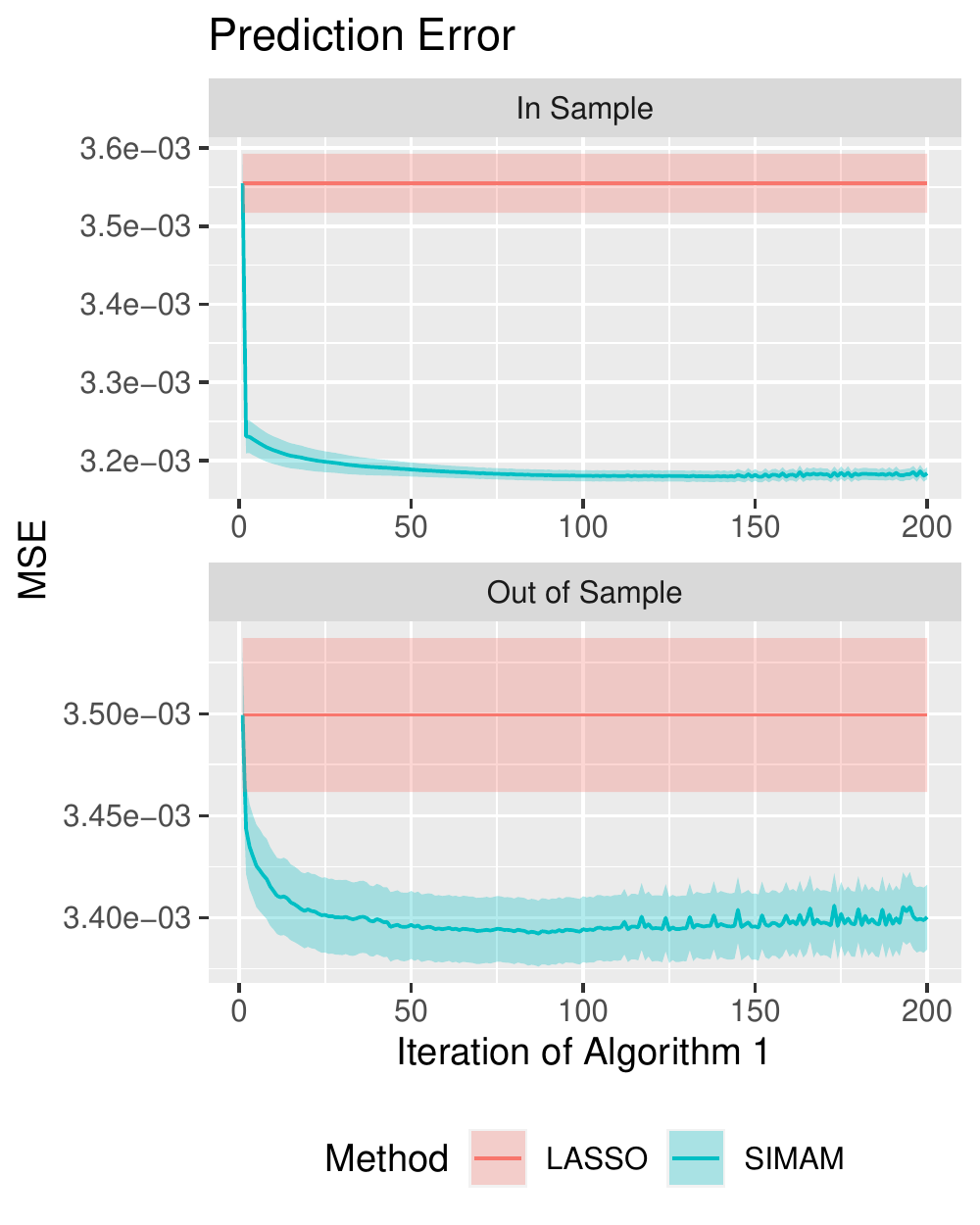}
        \caption{Time series with Uniform$(-0.1,0.1)$ noise}\label{fig: uniform9dim}
      \end{subfigure}
      \caption{Prediction Mean Squared Error (MSE) by LASSO and SIMAM for multivariate time series with $M = 9$ nodes, $900$ samples for training (in-sample) and $100$ samples for testing (out-of-sample). The left and right panels are from data generated by a SIMAM model respectively with Gaussian N$(0,0.05^2)$ noise and Uniform$(-0.1,0.1)$ noise. 
      For the SIMAM model, we use \cref{alg: SIMAM} with the LASSO solution as a warm start.
      Note that in all the panels, the $x$-axis denotes the iteration step for SIMAM in \cref{alg: SIMAM}, the prediction MSE by LASSO does not change along these iterations, making it a horizontal line. The experiment is repeated on $50$ independently generated time series data with different random seeds.}
      \label{fig: sim_9dim}
\end{figure}

In fact, as the dimension (node size) increases, such superior performance of SIMAM compared with LASSO still exists. 
In \cref{fig: sim_36dim}, we increase the dimension from $M = 9$ to $M = 36$, with the number of time points unchanged. We use the same set of monotone functions to introduce the nonlinear structure: for the $j$-th node, we generate the data with monotone function $f_l(x)$ in \cref{eq: sim funcs} where $l = [j \mod 9] +1, ~ \forall j\in [36]$. The true sparsity level in this case changes to $s_j^* = 6$. We still consider two types of noises:
(a) Gaussian noise from N$(0,0.05^2)$; 
(b) Bounded Uniform noise from Uniform$(-0.1,0.1)$.
To solve SIMAM in this case, the hard-threshold levels increases to $s_j =8$ accordingly. The step-size is still $0.01$, and we start \cref{alg: SIMAM} with the corresponding LASSO solutions. As shown in \cref{fig: sim_36dim}, when the node size has been enlarged, SIMAM still has a better prediction performance than LASSO in terms of both training and test prediction error. 

\begin{figure}[ht!]
    \begin{subfigure}[b]{.48\linewidth}
        \centering
        \includegraphics[scale = 0.7]{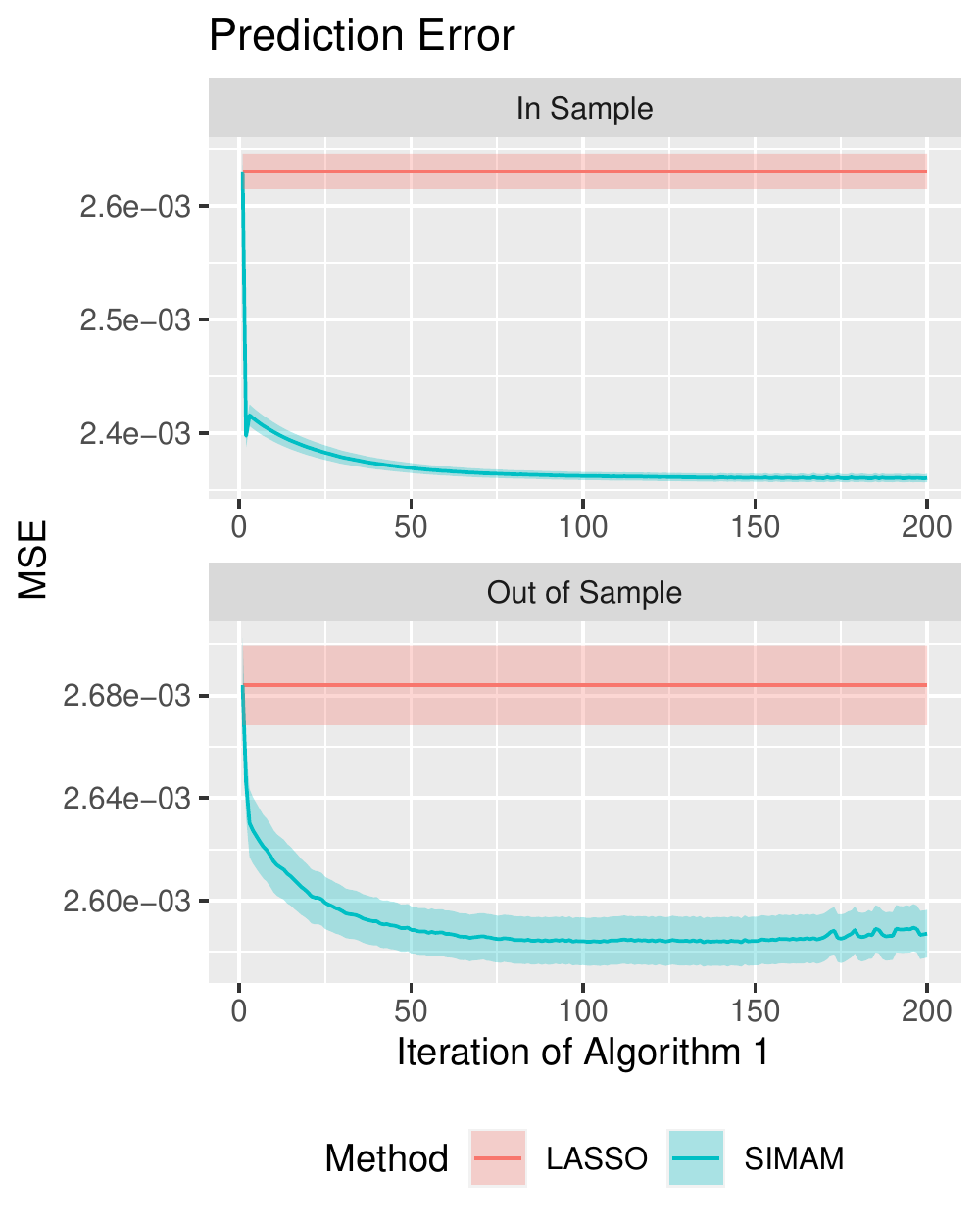}
        \caption{Time series with N$(0, 0.05^2)$ noise}\label{fig: gaussian36dim} 
    \end{subfigure}
    \begin{subfigure}[b]{.48\linewidth}
        \centering
        \includegraphics[scale = 0.7]{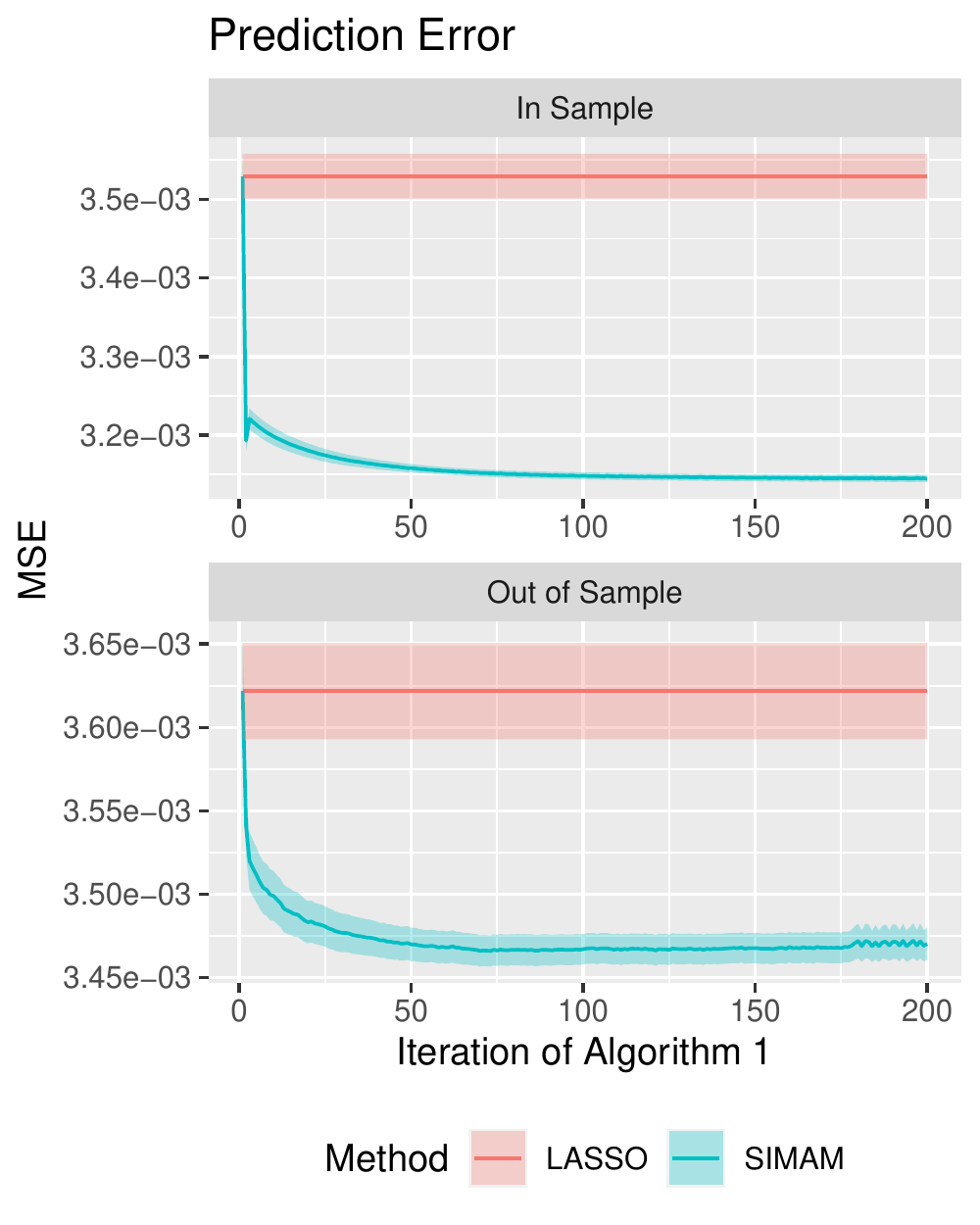}
        \caption{Time series with Uniform$(-0.1,0.1)$ noise}\label{fig: uniform36dim}
      \end{subfigure}
      \caption{Prediction MSE by LASSO and SIMAM for multivariate time series with a larger node size $M = 36$ compared to \cref{fig: sim_9dim}. Other settings are the same as \cref{fig: sim_9dim}.}
      \label{fig: sim_36dim}
\end{figure}

\section{Real Data Examples}

We validate our methodology and the main hypothesis on the Chicago crime data set and the MemeTracker data set. One challenge of real-data network estimation is the validation since there is no obvious ground truth. For both applications, we provide two types of validations: 
\begin{enumerate}
    \item Out-of-sample prediction performance showing that SIMAM fits the real data well without underfitting or overfitting compared to other popular learning methods;
    \item External knowledge of the influence network that are not included in the data set when training the model. For the Chicago crime example, we use the geographical information of communities to validate our learned community clusters. For the MemeTracker example, we use `time lag' and `pct of top quotes' indices for media influence to evaluate the top influential media sites we learned.
\end{enumerate}

\subsection{Chicago Crime Data}

We begin by seeing how our method performs on inferring a crime network based on records of Chicago crimes. Inferring the patterns of the crime locations over time and predicting the number of crime events in different areas could help the police forces work better and provide on-time security information for residents \cite{cmc.2019.06433}, which has been investigated by a number of studies, including \cite{Stomakhin_2011,Mark_2019,zhou2018non}. In the following, we use the monotone SIMAM method to conduct inference and prediction for the Chicago crime data, comparing with other popular methods.

Specifically, the crime data in 77 pre-defined community areas of Chicago from Jan 2004 to Dec 2018 are collected, focusing on severe types of crimes in each area (including homicide and battery) with a 2-day discretization.
The first $90\%$ of the data (before June 30, 2017) are regarded as the observed training set, while the remaining $10\%$ (final 18-month) data from July 2017 to Dec 2018 would serve as an out-of-sample test set to evaluate the prediction performance. To learn the SIMAM model, among the training set with $2465$ time points, the last $1/10$ part would be held out for tuning a good time $K_j$ ($K_j\leq 500$) for \cref{alg: SIMAM} to stop the iteration in the $j$-th area, where $j \in [M], ~p = 77$. Besides, we use $0.02$ as the step size, and cross validated sparsity from LASSO as the sparsity levels for hard-threshold in \cref{alg: SIMAM}.

We then obtain the estimated coefficient matrix $\widehat{A} = (\hat{u}_1,\dots, \hat u_M)$ that contains the information of the influence network, and the nonlinear monotone functions $\hat{f}_j, j=1,\dots,77$. In \cref{function_plot} some example monotone functions extracted from the Chicago crime data are given, indicating that there exist highly nonlinear and heterogeneous structure.

\begin{figure}[ht!]
    \centering
    \input{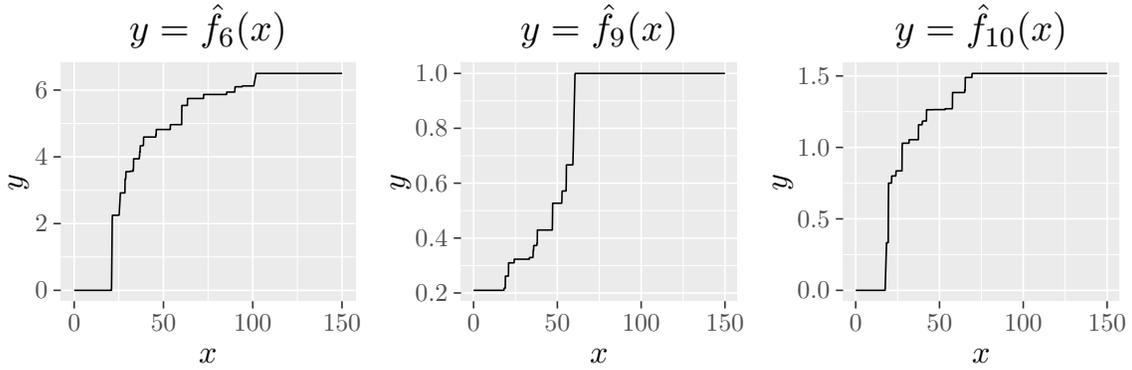}
    \caption{Example monotone functions extracted from Chicago crime data using \cref{alg: SIMAM}}
    \label{function_plot}
\end{figure}

In terms of predicting the number of the crime events in different areas, we observe a significant improvement by using monotone SIMAM, compared with the popular high dimensional multi-variate autoregressive methods including VAR (vector autoregressive model), LASSO (VAR with $\ell_1$ penalty), and PAR LASSO (generalized Poisson autoregressive model with $\ell_1$-regularization). Specifically, after constructing the above models based on the training set, we predict the crime counts in the test set, and then calculate the out-of-sample prediction RMSEs for each of the 77 pre-defined Chicago areas. 

To see whether SIMAM improves the prediction performance, since the crime counts from different Chicago areas are highly heteroscedastic, a paired t-test among the prediction RMSEs over the $77$ areas is recommended to compare the out-of-sample prediction performances. 
As shown in \cref{fig: crime_pred}, the differences of RMSEs in all $77$ areas between any baseline method and SIMAM tend to be larger than zero. 
More rigorously, $3$ sets of paired t-tests are conducted, each with the null hypothesis: the out-of-sample prediction RMSEs of SIMAM are not smaller than those of the baseline method in group 2 in \cref{table: chicago crime}:
\begin{equation}
    \text{H}_0:~ \bbE\left[\text{RMSE}_j({\text{Method 1}}) - \text{RMSE}_j({\text{Method 2}})\right] \geq 0, ~ j\in [77].
\end{equation}
The p-values in \cref{table: chicago crime} for the null hypotheses are respectively $0.0317$, $0.071$ and $8\times 10^{-9}$ , indicating that the out-of-sample prediction RMSEs from SIMAM are significantly smaller than those from PAR LASSO (Poisson Autoregressive model with $\ell_1$ penalty), LASSO and VAR.

\begin{figure}[ht!]
    \centering
    \includegraphics[width=0.98\linewidth]{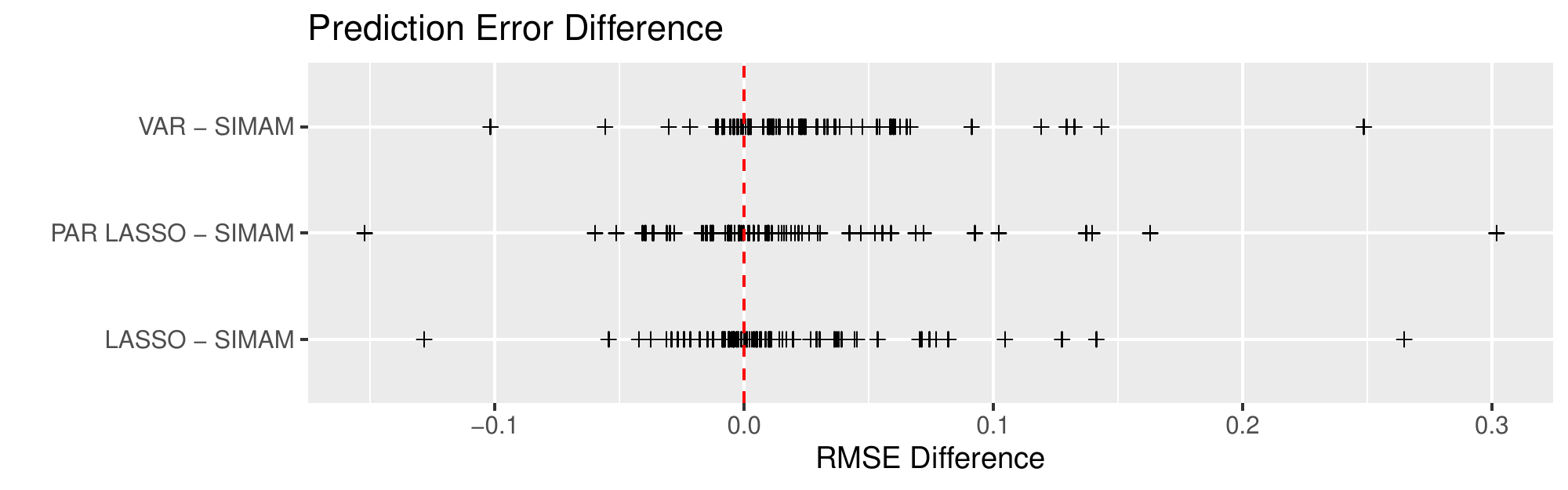}
    \caption{The prediction RMSE differences between baseline methods and SIMAM in out-of-sample Chicago crime data. For each of the $77$ pre-defined Chicago areas, by using VAR, PAR LASSO, LASSO and SIMAM method, the crime counts are predicted and the prediction errors in the test set are calculated in terms of RMSE. We take the difference between errors from the baseline methods and SIMAM for each area. In the figure, all these $77$ prediction RMSE differences for each baseline method are given.} 
    \label{fig: crime_pred}
\end{figure}

\begin{table}[ht!]
 \centering
 \begin{tabular}{lllllc}
 \hline
 Chicago Crime & Method 1 & Method 2 & p-value  & p.signif & alternative \\ 
 \hline
 RMSE & SIMAM & PAR LASSO & 0.0317  & * & $<$ \\ 
 RMSE & SIMAM & LASSO & 0.00709  & ** & $<$ \\
 RMSE & SIMAM & VAR & 8e-9 &  **** & $<$ \\ 
 \hline
 \end{tabular}
 \caption{Paired t-tests for prediction RMSEs in the out-of-sample Chicago crime data. As shown in the table, the p-values for the null hypotheses that SIMAM doesn't have a lower prediction RMSE than PAR LASSO, LASSO and VAR are all smaller than $0.05$, indicating that SIMAM has a significantly lower prediction error in the test set than the others.}
 \label{table: chicago crime}
\end{table}

To infer the patterns of crime locations over time, we conduct a spectral clustering based on the coefficient matrix $\widehat{A}$ derived from the above SIMAM procedure. Specifically, we first transform the coefficient matrix to an undirected adjacency matrix $\widetilde{A}$, by replacing all positive entries in the coefficient matrix with 1 and otherwise with 0, followed by a symmetrization with the \textit{or}-operator:
\begin{equation}
    \widetilde{A}_{ij} = \bigg\{\begin{array}{ll}
        1, & \text{if } \widehat{A}_{ij} >0 \text{ or } \widehat{A}_{ji} >0; \\
        0, & \text{else.}
    \end{array}
\end{equation}
Therefore $\widetilde{A}_{ij} = 1$ means that the crime events in the $i$-th area influence or is influenced by the $j$-th area.
Having acquired the adjacency matrix $\widetilde{A}$, we apply the standard spectral clustering algorithm with cluster number $K_{cluster} = 4$ (see \eg \cite{Rohe_2011, zhou2018non}). 
Hence, we obtain a block clustering for the patterns of crime locations in Chicago only according to our estimated coefficient matrix. 
The result is shown in \cref{fig: chicago_map}, with colors indicating cluster membership.
Note that without any knowledge of the geographical information such as latitudes and longitudes in the data, the clustering results based on the SIMAM estimated coefficients have clear patterns that conform with actual geographical locations, providing some validation to the estimated influences of crime events among these areas.

\begin{figure}[ht!]
 \centering
 \includegraphics[width = 0.98\linewidth]{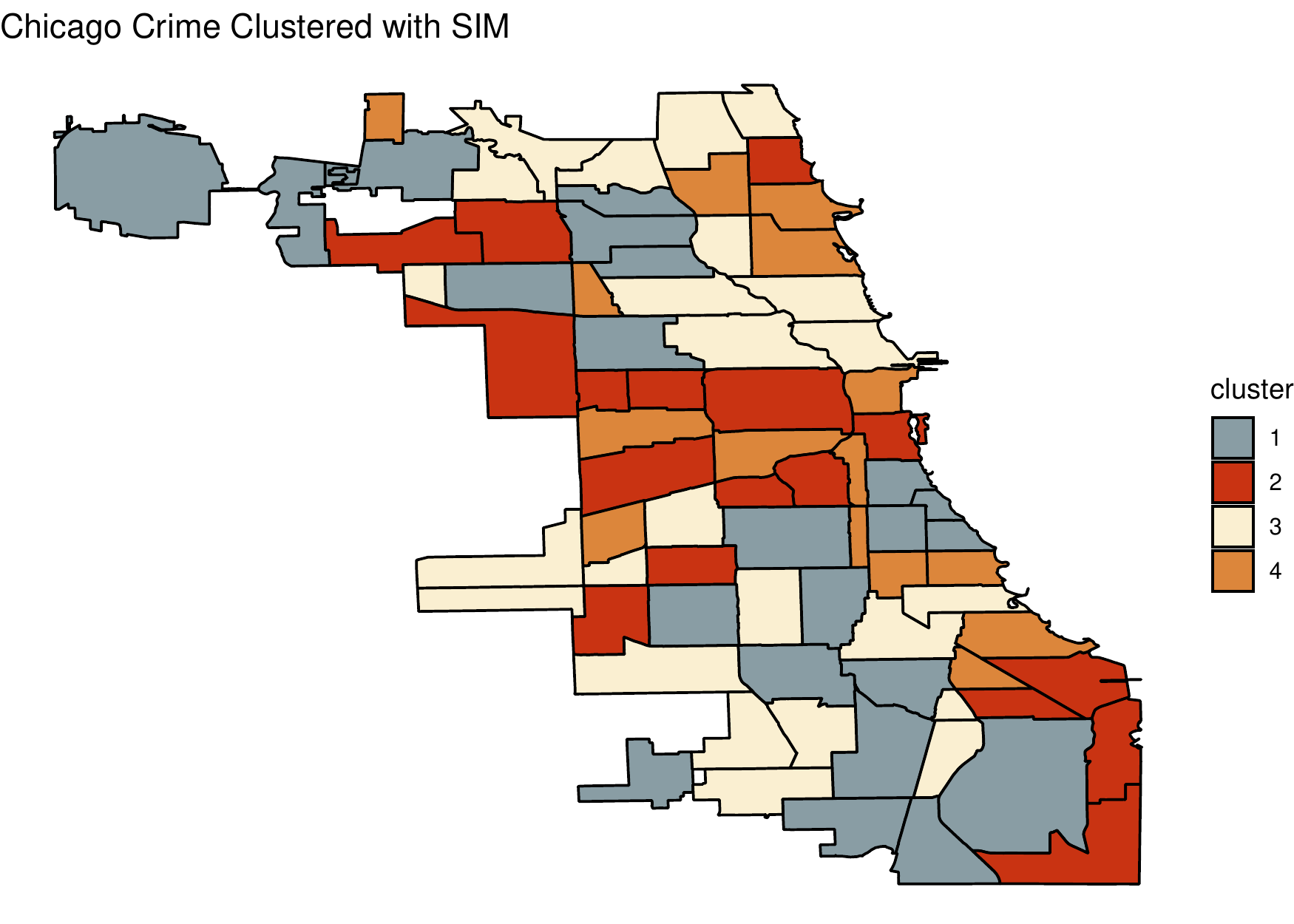}
 \caption{Clusters learned from crime data using SIMAM. The clusters are overlaid on a map of community areas in Chicago. No geospatial information is provided in the data, but clusters show geographical patterns.}
 \label{fig: chicago_map}
\end{figure}

\subsection{MemeTracker: social media data}

An interest in social networks is to infer the influence network among media sources, based on which one can have a better sight of the information flow. We collect news event data for 197 sources of media in a period from August 2008 to December 2009, with one-hour discretization. Thus the total sample size of such time series is 3602
\citep{Mark_2019}. 

We use popular methods including LASSO, PAR LASSO and VAR to analyze this MemeTracker data set as baseline methods to compare with our proposed monotone SIMAM. 

When we implement \cref{alg: SIMAM} for SIMAM, the maximum number of iterations is set to be 100, up to which we have observed a good performance in convergence and accuracy. Step sizes $\eta_j, ~j\in [197]$ are equally fixed as $0.05$, according to pre-experiment. To avoid underfitting or overfitting, we split the data into three parts with the proportions $8:1:1$ (training : validation: testing), through which the iteration stop time is tuned separately for each of the $197$ media source based on the prediction performance on the validation set. The sparsity levels for hard-thresholding are estimated by the cross validation procedure in LASSO, \ie for each $j\in [197]$, we use the $\ell_0$ norm of the LASSO solution under the sparsity parameter chosen by cross validation as the sparsity threshold $s_j$.

Once we have constructed the all the abovementioned models within the training set, we calculate their prediction RMSEs on the out-of-sample data for each of the $197$ media sources. Since the post numbers from different media sources might be highly heteroscedastic, a paired t-test is recommended to test the differences of prediction RMSEs between any baseline method and SIMAM.
The out-of-sample prediction RMSE differences are visualized in \cref{fig: meme_pred}. Following that, three paired t-tests are conducted with null hypotheses that SIMAM has no smaller prediction RMSEs than `Method 2' in \cref{table: meme_predict}. 
The p-values as shown in \cref{table: meme_predict} are all smaller that $10^{-4}$, indicating a very strong evidence that SIMAM has smaller prediction RMSEs than other methods.

\begin{figure}[ht!]
 \centering
 \includegraphics[width = 0.98\linewidth]{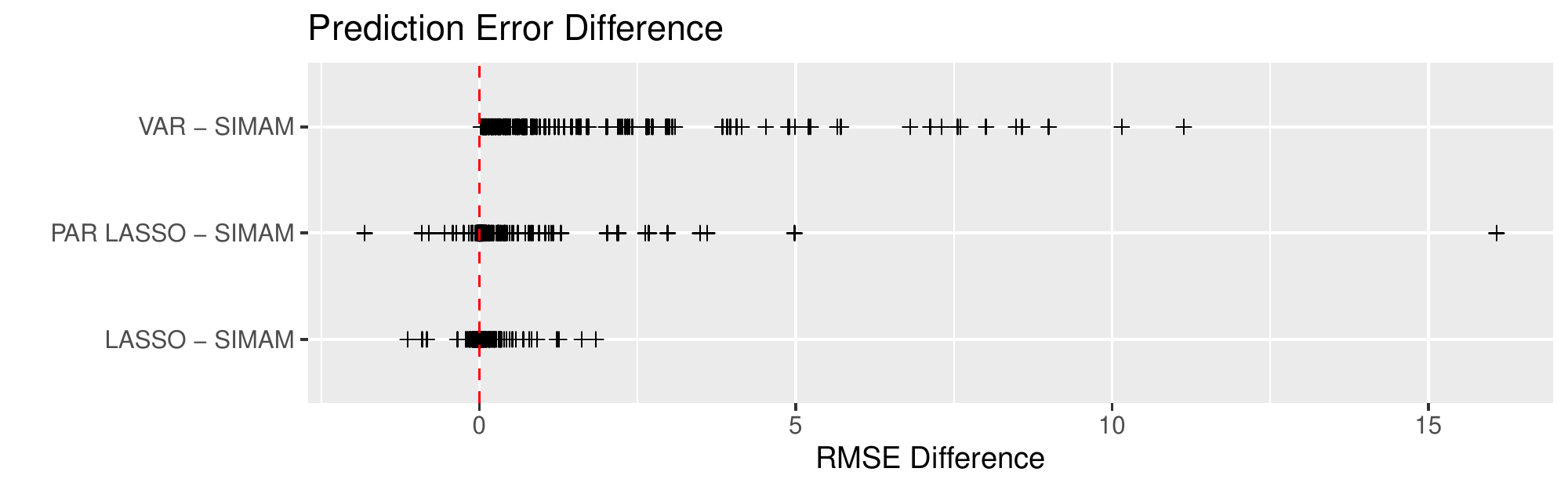}
 \caption{The prediction RMSE difference between baseline methods and SIMAM in out-of-sample MemeTracker data. For each of the $197$ media sources, by using VAR, PAR LASSO, LASSO and SIMAM method, the crime counts are predicted and the prediction errors in the test set are calculated in terms of RMSE. In the figure, all these $197$ prediction RMSE differences for each baseline method are given.}
 \label{fig: meme_pred}
\end{figure}

\begin{table}[ht!]
 \centering
 \begin{tabular}{lllllc}
 \hline
 MemeTracker & Method 1 & Method 2 & p-value  & p.signif & alternative \\ 
 \hline
 RMSE & SIMAM & LASSO & 5.06e-04  &*** & $<$ \\ 
 RMSE & SIMAM & PAR LASSO & 3.56e-06  & **** & $<$ \\ 
 RMSE & SIMAM & VAR & 7.72e-20 & **** & $<$ \\ 
 \hline
 \end{tabular}
 \caption{Paired t-tests for prediction RMSEs in the out-of-sample MemeTracker data. As shown in the table, the p-values for the null hypotheses that SIMAM doesn't have a lower prediction RMSE than PAR LASSO, LASSO and VAR are all smaller than $0.05$, indicating that SIMAM has a significantly lower prediction error in the test set than the others.}
 \label{table: meme_predict}
\end{table}

Since the element of the coefficient matrix ${A}_{ij}$ indicates the influence of the $i$-th media source on the $j$-th media source for any $i,j\in [197]$, 
the $i$-th row sum of the coefficient matrix ${A}$ can measure the overall influence of the $i$-th media source upon the entire investigated social media network. 
Therefore, to examine the structure of our learned coefficient matrix $\widehat{A}^{\text{SIMAM}}$ using SIMAM compared with $\widehat{A}^{\text{LASSO}}$ using LASSO, we add up each row of $\widehat{A}^{\text{SIMAM}}$ and $\widehat{A}^{\text{LASSO}}$.
We then rank the influences of the $i$-th media source respectively by $\sum_{j=1}^{197} \widehat{A}^{\text{SIMAM}}_{ij}$ and $\sum_{j=1}^{197} \widehat{A}^{\text{LASSO}}_{ij}$.

We therefore extract the top $20$ media sources with the first $20$ largest row sums among all $197$ sources, using both Lasso and SIMAM.
To see whether these learned top influential media sources are indeed influential, we use some external knowledge of these media sources provided in \url{http://snap.stanford.edu/memetracker/}. The measures of influence for media sources include:
\begin{enumerate}
 \item \textit{Time lag}: number of hours between the time a media site first reported a story and when the story (quote) reached its peak. Negative times mean that site reported the news before it reached its peak, and positive numbers mean that the site was lagging and only reported the news after it reached its peak.
 \item \textit{Pct of top quotes}: Fraction of top stories (quotes) the site covered. The higher the number, the more important news was covered by the site.
\end{enumerate}

Hence, to compare the actual influence of the top 20 media sources learned respectively from SIMAM and LASSO, we look up their `time lag' and `pct of top quotes' indices in the MemeTracker website.
As shown in \cref{table: meme_rank}, the top $20$ media sources extracted by SIMAM have significantly lower negative time lags, 
and higher fractions of top quotes on average than LASSO. Therefore the media sources learned by SIMAM are more `influential', in the sense that they tend to report hot stories earlier and have more top stories (quotes) covered than the ones learned from LASSO. 
This also gives some validation that SIMAM can help us to have a better sight of  the information flow in the social media network.

\begin{table}[ht!]
 \centering
 \begin{adjustbox}{width=1\textwidth}
 \begin{tabular}{clrr|lrr}
 \hline
 &\multicolumn{3}{c}{\textbf{SIMAM}} & \multicolumn{3}{c}{\textbf{LASSO}}\\
 \hline
 rank & \textbf{media source} &\textbf{time lag}&\textbf{pct of top}&\textbf{media source}&\textbf{time lag}&\textbf{pct of top}\\ 
 \hline
 1 & cbc.ca & -7.00 & 31 & thinkprogress.org & -6.50 & 30 \\ 
 2 & cnn.com & -16.50 & 54 & washingtonmonthly.com & -11.50 & 34 \\ 
 3 & blogs.wsj.com & -10.50 & 32 & hotair.com & -26.50 & 42 \\ 
 4 & ctv.ca & -6.50 & 62 & features.csmonitor.com & -0.50 & 32 \\ 
 5 & instablogs.com & -10.00 & 80 & blogs.wsj.com & -10.50 & 32 \\ 
 6 & clkurl.com & -11.00 & 81 & cnsnews.com & -4.00 & 53 \\ 
 7 & npr.org & -5.00 & 43 & kctv5.com & 2.00 & 30 \\ 
 8 & wnbc.com & -2.00 & 30 & cbs46.com & -2.00 & 31 \\ 
 9 & iht.com & -10.00 & 78 & wbbm780.com & -6.00 & 31 \\ 
 10 & bostonherald.com & -9.00 & 74 & wnbc.com & -2.00 & 30 \\ 
 11 & economy.finance.com & -7.00 & 31 & capitolhillblue.com & 3.00 & 46 \\ 
 12 & wral.com & -10.00 & 55 & primebuzz.kcstar.com & -6.00 & 41 \\ 
 13 & abcnews.go.com & -10.50 & 58 & cnnpoliticalticker.wordpress & -19.50 & 56 \\ 
 14 & huffingtonpost.com & -18.00 & 73 & cbc.ca & -7.00 & 31 \\ 
 15 & usnews.com & -5.50 & 44 & themonitor.com & -2.00 & 49 \\ 
 16 & afp.google.com & -8.00 & 48 & fresnobee.com & -6.50 & 42 \\ 
 17 & ap.google.com & -10.00 & 79 & unionleader.com & -9.00 & 31 \\ 
 18 & forums.somd.com & -4.00 & 35 & afp.google.com & -8.00 & 48 \\ 
 19 & hotair.com & -26.50 & 42 & wcpo.com & -5.50 & 46 \\ 
 20 & macleans.ca & -10.00 & 67 & voteoften.us & -8.00 & 47 \\ 
 \hline
 {}& \textbf{Mean} & \textbf{-9.85}&\textbf{54.85}&{}&\textbf{-6.8}&\textbf{39.1}\\
 \hline
 \end{tabular}
 \end{adjustbox}
 \caption{Top $20$ influential media sources learned by SIMAM and LASSO. In the table, the column `media source' contains the top $20$ media sources learned respectively by SIMAM and LASSO; we evaluate the quality and influence of these learned media sources by two external indices: `time lag' contains the corresponding time lag indices for each media source, the smaller the number, the earlier the media source report the hot news; `pct of top' represents the percentage of top stories (quotes) a media source covered, the larger the number, the more important news was covered. Both `time lag' and `pct of top' are not provided in the data.}
 \label{table: meme_rank}
\end{table}

\section{Proof Overview}

One of the major challenges to prove the theoretical results comes from the unknown non-linear structures in the monotone SIMAM. 
Compared to prior methods such as the AR model and GLAR model, whose proofs rely heavily on the parametric assumptions and the constraints of the parameters, SIMAM introduces a set of highly non-linear functions that are unknown, making it hard to find analytical solutions and optimize simultaneously over the direction vectors and the functions. The theoretical guarantees for the estimator from \cref{alg: SIMAM} are more challenging to derive than the previous MLE type of estimators, since the optimization of the unknown monotone function changes each time after the direction vector is updated, and vice versa for estimating the direction vector. To tackle this challenge, we utilize the contractiveness property of isotonic functions (\cref{lemma: contractive}) and the bounded complexity of the reference space (\cref{lemma: permutation}), so as to give the \emph{uniform} bound over all possible isotonic reference vectors to control the deviation regardless of the changes in the alternating procedure (see \cref{lemma: sup_2_norm}). 

Another technical challenge comes from the dependence structure of the data, which makes the proofs for single index models in the i.i.d case no longer suitable. Based on the conditions analogous to the restricted strong convexity implicitly included in the identifiability assumption (\cref{eq: identifiability}), and the mild assumption for the noise (\cref{asspt: noise}), we use the martingale concentration inequalities to prove the non-asymptotic results for the dependent data.

In the following, we provide the major steps to prove \cref{thm: coefficient conv}, while the mathematical details and complete proofs for other results would be deferred to the appendix (\cref{section: appendix}).
\subsection{Iterative Guarantee}

To prove \cref{thm: coefficient conv}, we first show that after each iteration step, the updated estimator $u_j^{(k)}$ get closer to the true parameters $u_j^*$ compared to $u_j^{(k-1)}$ in a linear convergence manner with a remainder term, for any $j\in [M],~ k\geq 1$.
\begin{lemma}
    For any $j \in [M]$ and iteration step $k \in \bbN^+$, suppose the assumptions \cref{asspt: mono_Lip}, \cref{eq: identifiability} and \cref{asspt: data_bound} are satisfied, as long as $\langle u_j^{(k-1)},u_j^* \rangle \geq 0$, we have
    \begin{equation}
        (1-\sqrt{\frac{s_j^{*}}{s_j}} - \frac{\alpha_j \delta_j}{L_j^2\beta})\parallel u_j^{(k)}- u_j^{*}\parallel_2^2 \leq (1-\frac{\alpha_j}{L_j^2\beta}) \parallel u_j^{(k-1)} - u_j^{*} \parallel_2^2 +\frac{\epsilon_j^2}{L_j^2\beta} +\frac{C^2(Z,u_j^{(k-1)})}{\delta_j\alpha_j\beta}
    \end{equation}
    for any positive scalar $\delta_j>0$,  with the numerator of the last term formulated as
    \begin{equation}
        \begin{split}
            C(Z_j, u_j^{(k-1)}) =& \frac{\sqrt{2s_j+s_j^*}}{N}\parallel \X_{-N}^T \left(\X_{-0,j} - f_j^{*}(\X_{-N}\cdot u_j^{*})\right) \parallel_{\infty} \\
            +& \sqrt{\beta/N} \parallel \iso_{\bfv_{k-1}} (f_j^{*}(\X_{-N}\cdot u_j^{*})) - \iso_{\bfv_{k-1}}(\X_{-0,j})
            \parallel_2,\\
        \end{split}
    \end{equation}
    where $\bfv_{k-1} = \X_{-N}\cdot u_j^{(k-1)}$.
    \label{lemma: iterative bound}
\end{lemma}

\subsection{Bounding the remainder term}

To further control $C(Z_j,u_j^{(k-1)})$ in the remaining term for every iteration step $k = 1,2,\dots$ which determines the statistical error bound, the following two lemmas are provided that can reveal some rationales behind \cref{alg: SIMAM}. 

\begin{lemma}
   With the martingale difference assumption \cref{asspt: noise} and the data boundedness assumption \cref{asspt: entry_bound} satisfied, for any $j \in [M]$, we have
   \begin{equation}
    \frac{1}{M}\parallel \X_{-M}^T \left(\X_{-0,j} - f_j^{*}(\X_{-M}\cdot u_j^{*})\right) \parallel_{\infty} \leq \frac{\sigma_j M_x}{\sqrt M}\sqrt{2 \log(\frac{2M}{\gamma})} 
   \end{equation}
   with probability at least $1- \gamma$.
   \label{lemma: infty_norm_z}
\end{lemma}

\cref{lemma: infty_norm_z} essentially controls the magnitude of the martingale difference $\{X_{t,j} - f_j^*(X_{t-1} u_j^*)\}_{t=1}^T$ after projecting them on the space spanned by $\{X_{t-1}\}$.
The following lemma, which is more complicated than the former, controls the distance between the node observations and their corresponding conditional expectations after both are isotonically projected with respect to a reference vector $\bfv$. Since we desire to bound this term no matter how the reference vector $\bfv$ varies in the iteration procedure, we need to derive a uniform bound for all possible reference vectors.

\begin{lemma}
     For any $j\in [M]$, if the assumptions
     \cref{asspt: mono_Lip}, \cref{eq: identifiability}, \cref{asspt: noise} and \cref{asspt: entry_bound} are satisfied, 
    the following 2-norm distance between isotonically projected observations $\X_{-0,j}$ and their corresponding conditional means $f_j^*(\X_{-T} u_j^*)$ can be uniformly bounded over the set $\bfV= \{ \bfv = \X_{-T} \bfu \in \bbR^{T}:  \bfu\in \mathcal{S}^{M-1} \text{with sparsity } s_j\}$ with probability at least $1-2\gamma$,
    \begin{equation}
        \bfv
    \end{equation}
    \begin{equation}
        \begin{split}
            &\sup_{\bfv\in \bfv} \frac{1}{\sqrt{T}}\|\iso_{v} (f_j^{*}(\X_{-T}u_j^{*})) - \iso_{v}(\X_{-0,j})\|_2
            \\
            \leq & 4 T^{-\frac{1}{3}} \left[
                \left(2\sqrt{2\sigma_j^2 \log \frac{2T^{2s_j}M^{s_j}}{\gamma} } + 2 L_j M_x \sqrt{s_j^{*}} \right)
                \sigma_j^2 \log( \frac{T^{2s_j}(T+1)M^{s_j}}{\gamma})
            \right]^{\frac{1}{3}} \\
            &+ \frac{2\sigma_j}{\sqrt{T}} \sqrt{\log\left(
                \frac{T^{2s_j}(T+1) M^{s_j}}{\gamma}
            \right)}.
        \end{split}
    \end{equation}
    \label{lemma: sup_2_norm}
\end{lemma}

With \cref{lemma: infty_norm_z} and \cref{lemma: sup_2_norm}, we could then get the non-asymptotic uniform bound for the remaining term $C(Z_j,u_j)$ over all possible $u_j$, the rate of which would determine the statistical error bound in the main result.

\begin{corollary}\label{lemma: def_Delta}
    Denote $\mathbb{B}_0(s_j) = \{u\in \bbR^M: \|u\|_2 = 1, \|u\|_0 = s_j \}$ as the set of all direction vectors with sparsity level $s_j$. Suppose assumptions listed in \cref{lemma: infty_norm_z} and \cref{lemma: sup_2_norm} are all satisfied, then with probability at least $1- 3\gamma$, we have
    \begin{equation}\label{eq: def_Delta}
        \begin{split}
            &\sup_{u\in \mathbb{B}_0(s_j)} C(Z_j,u) \leq  \Delta_j\\
             := & T^{-\frac{1}{2}} \sigma_j \left[
                 M_x \sqrt{2(2s_j+s_j^*) \log(\frac{2M}{\gamma})}
                 +2 \sqrt{\beta \log \left(\frac{T^{2s_j}(T+1)M^{s_j}}{\gamma}\right)}
             \right] \\
             &+
             4 T^{-\frac{1}{3}} \sqrt{\beta} \left[
                \left(2\sqrt{2\sigma_j^2 \log \frac{2T^{2s_j}M^{s_j}}{\gamma} } + 2 L_j M_x \sqrt{s_j^{*}} \right)
                \sigma_j^2 \log( \frac{T^{2s_j}(T+1)M^{s_j}}{\gamma})
            \right]^{\frac{1}{3}}
             \\
            =& O\left( 
            T^{- \frac{1}{3}}
            \left[
                \log\left(
                \frac{T^{2s_j+1}M^{s_j}}{\gamma}
            \right)
            \right]^{\frac{1}{2}}
            \right).
        \end{split}
    \end{equation}
\end{corollary}

\subsection{Mathematical Induction}

Based on the above lemmas, we use the mathematical induction to complete the proof of \cref{thm: coefficient conv}.
\begin{proof}
    Given a good initialization $u_j^{(0)}$ such that $\langle u_j^{(0)}, u_j^*\rangle\geq 0
    $, we know that 
    \begin{equation}
        \| u_j^{(0)} - u_j^*\|_2^2 = \| u_j^{(0)}\|_2^2 + \|u_j^*\|_2^2 - 2 \langle u_j^{(0)}, u_j^*\rangle \leq 2.
        \label{eq: induct1}
    \end{equation}

    Recall the definition 
    \begin{equation}
        \theta_j = \frac{1 - \frac{\alpha_j}{L_j^2 \beta}}{1 - \frac{\delta_j\alpha_j}{L_j^2\beta} - \sqrt{\frac{s_j^{*}}{s_j}}} ~\text{for some } \delta_j>0,~
           \text{and}~ R_j = \frac{\frac{\epsilon_j^2}{L_j^2\beta} +\frac{\Delta_j^2}{\delta_j\alpha_j\beta}} { (1-\delta_j) \frac{\alpha_j}{L_j^2\beta} -\sqrt{\frac{s_j^*}{s_j}}}.
    \end{equation}

    Now we use mathematical induction to prove the result for any $j\in [M]$. By \cref{eq: induct1}, when $k = 0$, $\| u_j^{(k)} - u_j^*\|_2^2 \leq 2 \theta_j^{k} + (1-\theta_j^k)R_j$ is true. For $K \geq 1$,
    assume 
    \begin{equation}
        \| u_j^{(k)} - u_j^*\|_2^2 \leq 2 \theta_j^{k} + (1-\theta_j^k)R_j
        \label{eq: induct_asspt}
    \end{equation}
    holds for all $k\leq K-1$.
    First, we want to show that for these $k\leq K-1$, once the induction assumption (\cref{eq: induct_asspt}) is true, the angle between these estimators and the truth is acute.  With the sparsity condition (\cref{asspt: sparsity}) satisfied, we have that $R_j \leq 2$. 
    Therefore,
    \begin{equation}
        \begin{split}
            \langle u_j^{(k)} , u_j^* \rangle 
            &= \frac{1}{2} \left(\|u_j^{(k)}\|_2^2 + \| u_j^*\|_2^2 - \| u_j^{(k)} - u_j^*\|_2^2 \right) \\
            &\geq 1 - 
            \frac{1}{2}\left[2\theta_j^k +(1-\theta_j^k)R_j \right]\\
            &\geq 1 - 
            \frac{1}{2}\left[2\theta_j^k +
            2(1-\theta_j^k) \right] =0.
        \end{split}
    \end{equation}

    With this condition satisfied for all $k\leq K-1$, by \cref{lemma: iterative bound} and \cref{lemma: def_Delta} we have
    \begin{equation}
        \begin{split}
            \|u_j^{(K)} - u_j^*\|_2^2 
            &\leq \theta_j \|u_j^{(K-1)} - u_j^*\|_2^2 + \frac{\frac{\epsilon_j^2}{L_j^2\beta} +\frac{\Delta_j^2}{\delta_j\alpha_j\beta}}{1 - \frac{\delta_j\alpha_j}{L_j^2\beta} - \sqrt{\frac{s_j^{*}}{s_j}}} \\
            &= \theta_j \|u_j^{(K-1)} - u_j^*\|_2^2 +(1-\theta_j) R_j\\
            &\leq \theta_j \left[
                2 \cdot \theta_j^{K-1} + (1-\theta_j^{K-1})\cdot R_j
            \right] + (1-\theta_j)R_j\\
            &= 2\theta_j^{K} + (1-\theta_j^{K}) R_j.
        \end{split}
        \label{eq: recurs}
    \end{equation}
    Hence, the induction reveals that 
    \begin{equation}
        \| u_j^{(K)} - u_j^*\|_2^2 \leq 2 \theta_j^{K} + (1-\theta_j^K)R_j \leq 
         2 \theta_j^{K} + R_j 
    \end{equation}
    is true for any 
    $K\in \bbN^{+}$.

\end{proof}

\section{Conclusion}

In this paper, we construct the monotone SIMAM to improve the prediction and influence network estimation for high-dimensional multi-variate time series or point process data. Such semi-parametric formulation gains more flexibility and robustness against model mis-specification due to its enlarged model space while retaining parametric models' desirable interpretation. Based on this model, we developed an alternating PGD algorithm for SIMAM (\cref{alg: SIMAM}) to estimate the underlying network structure, which takes the non-convex sparsity constraint in a high dimensional setting into account. Theoretically, using martingale concentration inequalities, we show that our algorithm converges in a geometric rate, and after sufficiently many iterations, the statistical error for the influence network estimation achieves the rate $O(T^{-\frac{1}{3}}\sqrt{s\log (TM)})$.
We also demonstrated that our algorithm for monotone SIMAM has a superior performance both on simulated data and two popular real data examples (Chicago crime data and MemeTracker social media data) compared to state-of-the-art parametric methods.

\newpage

\bibliography{references}

\begin{thebibliography}{56}
\providecommand{\natexlab}[1]{#1}
\providecommand{\url}[1]{\texttt{#1}}
\expandafter\ifx\csname urlstyle\endcsname\relax
  \providecommand{\doi}[1]{doi: #1}\else
  \providecommand{\doi}{doi: \begingroup \urlstyle{rm}\Url}\fi

\bibitem[Balabdaoui et~al.(2019)Balabdaoui, Durot, Jankowski,
  et~al.]{balabdaoui2019least}
Fadoua Balabdaoui, C{\'e}cile Durot, Hanna Jankowski, et~al.
\newblock Least squares estimation in the monotone single index model.
\newblock \emph{Bernoulli}, 25\penalty0 (4B):\penalty0 3276--3310, 2019.

\bibitem[Bellec et~al.(2018)]{bellec2018sharp}
Pierre~C Bellec et~al.
\newblock Sharp oracle inequalities for least squares estimators in shape
  restricted regression.
\newblock \emph{The Annals of Statistics}, 46\penalty0 (2):\penalty0 745--780,
  2018.

\bibitem[Blumensath and Davies(2009)]{blumensath2009iterative}
Thomas Blumensath and Mike~E Davies.
\newblock Iterative hard thresholding for compressed sensing.
\newblock \emph{Applied and computational harmonic analysis}, 27\penalty0
  (3):\penalty0 265--274, 2009.

\bibitem[Brown et~al.(2004)Brown, Kass, and Mitra]{brown2004multiple}
Emery~N Brown, Robert~E Kass, and Partha~P Mitra.
\newblock Multiple neural spike train data analysis: state-of-the-art and
  future challenges.
\newblock \emph{Nature neuroscience}, 7\penalty0 (5):\penalty0 456--461, 2004.

\bibitem[Canova(1995)]{canova1995vector}
Fabio Canova.
\newblock Vector autoregressive models: specification, estimation, inference
  and forecasting.
\newblock \emph{Handbook of applied econometrics}, 1:\penalty0 73--138, 1995.

\bibitem[Carroll et~al.(1997)Carroll, Fan, Gijbels, and
  Wand]{carroll1997generalized}
Raymond~J Carroll, Jianqing Fan, Irene Gijbels, and Matt~P Wand.
\newblock Generalized partially linear single-index models.
\newblock \emph{Journal of the American Statistical Association}, 92\penalty0
  (438):\penalty0 477--489, 1997.

\bibitem[Chatterjee et~al.(2015)Chatterjee, Guntuboyina, Sen,
  et~al.]{chatterjee2015risk}
Sabyasachi Chatterjee, Adityanand Guntuboyina, Bodhisattva Sen, et~al.
\newblock On risk bounds in isotonic and other shape restricted regression
  problems.
\newblock \emph{The Annals of Statistics}, 43\penalty0 (4):\penalty0
  1774--1800, 2015.

\bibitem[Chatterjee et~al.(2014)]{chatterjee2014new}
Sourav Chatterjee et~al.
\newblock A new perspective on least squares under convex constraint.
\newblock \emph{The Annals of Statistics}, 42\penalty0 (6):\penalty0
  2340--2381, 2014.

\bibitem[Chen and Samworth(2014)]{chen2014generalised}
Yining Chen and Richard~J Samworth.
\newblock Generalised additive and index models with shape constraints.
\newblock \emph{arXiv preprint arXiv:1404.2957}, 2014.

\bibitem[Cover(1967)]{cover1967}
Thomas~M. Cover.
\newblock The number of linearly inducible orderings of points in d-space.
\newblock \emph{SIAM Journal on Applied Mathematics}, 15\penalty0 (2):\penalty0
  434--439, 1967.
\newblock ISSN 00361399.
\newblock URL \url{http://www.jstor.org/stable/2946294}.

\bibitem[Dai et~al.(2021)Dai, Song, Barber, and Raskutti]{dai2021convergence}
Ran Dai, Hyebin Song, Rina~Foygel Barber, and Garvesh Raskutti.
\newblock Convergence guarantee for the sparse monotone single index model,
  2021.

\bibitem[Daniel Rivera~Ruiz(2019)]{cmc.2019.06433}
Alisha~Sawant Daniel Rivera~Ruiz.
\newblock Quantitative analysis of crime incidents in chicago using data
  analytics techniques.
\newblock \emph{Computers, Materials \& Continua}, 59\penalty0 (2):\penalty0
  389--396, 2019.
\newblock ISSN 1546-2226.
\newblock \doi{10.32604/cmc.2019.06433}.
\newblock URL \url{http://www.techscience.com/cmc/v59n2/27951}.

\bibitem[Dunsmuir(2015)]{dunsmuir2015generalized}
William~TM Dunsmuir.
\newblock Generalized linear autoregressive moving average models.
\newblock \emph{Handbook of Discrete-Valued Time Series. CRC Monographs}, 2015.

\bibitem[Durot(2002)]{durot2002sharp}
C{\'e}cile Durot.
\newblock Sharp asymptotics for isotonic regression.
\newblock \emph{Probability theory and related fields}, 122\penalty0
  (2):\penalty0 222--240, 2002.

\bibitem[Egesdal et~al.(2010)Egesdal, Fathauer, Louie, Neuman, Mohler, and
  Lewis]{egesdal2010statistical}
Mike Egesdal, Chris Fathauer, Kym Louie, Jeremy Neuman, George Mohler, and Erik
  Lewis.
\newblock Statistical and stochastic modeling of gang rivalries in los angeles.
\newblock \emph{SIAM Undergraduate Research Online}, 3:\penalty0 72--94, 2010.

\bibitem[Ertekin et~al.(2015)Ertekin, Rudin, McCormick,
  et~al.]{ertekin2015reactive}
{\c{S}}eyda Ertekin, Cynthia Rudin, Tyler~H McCormick, et~al.
\newblock Reactive point processes: A new approach to predicting power failures
  in underground electrical systems.
\newblock \emph{The Annals of Applied Statistics}, 9\penalty0 (1):\penalty0
  122--144, 2015.

\bibitem[Fokianos et~al.(2009)Fokianos, Rahbek, and
  Tj{\o}stheim]{fokianos2009poisson}
Konstantinos Fokianos, Anders Rahbek, and Dag Tj{\o}stheim.
\newblock Poisson autoregression.
\newblock \emph{Journal of the American Statistical Association}, 104\penalty0
  (488):\penalty0 1430--1439, 2009.

\bibitem[Foster et~al.(2013)Foster, Taylor, and Nan]{foster2013variable}
Jared~C Foster, Jeremy~MG Taylor, and Bin Nan.
\newblock Variable selection in monotone single-index models via the adaptive
  lasso.
\newblock \emph{Statistics in medicine}, 32\penalty0 (22):\penalty0 3944--3954,
  2013.

\bibitem[Fujita et~al.(2007)Fujita, Sato, Garay-Malpartida, Yamaguchi, Miyano,
  Sogayar, and Ferreira]{fujita2007modeling}
Andr{\'e} Fujita, Joao~R Sato, Humberto~M Garay-Malpartida, Rui Yamaguchi,
  Satoru Miyano, Mari~C Sogayar, and Carlos~E Ferreira.
\newblock Modeling gene expression regulatory networks with the sparse vector
  autoregressive model.
\newblock \emph{BMC systems biology}, 1\penalty0 (1):\penalty0 1--11, 2007.

\bibitem[Groeneboom and Hendrickx(2019)]{groeneboom2019estimation}
Piet Groeneboom and Kim Hendrickx.
\newblock Estimation in monotone single-index models.
\newblock \emph{Statistica Neerlandica}, 73\penalty0 (1):\penalty0 78--99,
  2019.

\bibitem[Guo et~al.(2017)Guo, Wu, and Yu]{guo2017time}
Hui Guo, Chaojiang Wu, and Yan Yu.
\newblock Time-varying beta and the value premium.
\newblock \emph{Journal of Financial and Quantitative Analysis}, 52\penalty0
  (4):\penalty0 1551--1576, 2017.

\bibitem[Hall and Willett(2015)]{hall2015online}
Eric~C Hall and Rebecca~M Willett.
\newblock Online learning of neural network structure from spike trains.
\newblock In \emph{2015 7th International IEEE/EMBS Conference on Neural
  Engineering (NER)}, pages 930--933. IEEE, 2015.

\bibitem[Hall et~al.(2016)Hall, Raskutti, and Willett]{hall2016inference}
Eric~C. Hall, Garvesh Raskutti, and Rebecca Willett.
\newblock Inference of high-dimensional autoregressive generalized linear
  models, 2016.

\bibitem[Hall et~al.(2018)Hall, Raskutti, and Willett]{hall2018learning}
Eric~C Hall, Garvesh Raskutti, and Rebecca~M Willett.
\newblock Learning high-dimensional generalized linear autoregressive models.
\newblock \emph{IEEE Transactions on Information Theory}, 65\penalty0
  (4):\penalty0 2401--2422, 2018.

\bibitem[H{\"a}rdle and Vieu(1992)]{hardle1992kernel}
Wolfgang H{\"a}rdle and Philippe Vieu.
\newblock Kernel regression smoothing of time series.
\newblock \emph{Journal of Time Series Analysis}, 13\penalty0 (3):\penalty0
  209--232, 1992.

\bibitem[Hristache et~al.(2001)Hristache, Juditsky, and
  Spokoiny]{hristache2001direct}
Marian Hristache, Anatoli Juditsky, and Vladimir Spokoiny.
\newblock Direct estimation of the index coefficient in a single-index model.
\newblock \emph{Annals of Statistics}, pages 595--623, 2001.

\bibitem[Hsu et~al.(2008)Hsu, Hung, and Chang]{2008ssvar_lasso}
Nan-Jung Hsu, Hung-Lin Hung, and Ya-Mei Chang.
\newblock Subset selection for vector autoregressive processes using lasso.
\newblock \emph{Computational Statistics \& Data Analysis}, 52\penalty0
  (7):\penalty0 3645--3657, 2008.
\newblock URL
  \url{https://EconPapers.repec.org/RePEc:eee:csdana:v:52:y:2008:i:7:p:3645-3657}.

\bibitem[Jain et~al.(2014)Jain, Tewari, and Kar]{jain2014iterative}
Prateek Jain, Ambuj Tewari, and Purushottam Kar.
\newblock On iterative hard thresholding methods for high-dimensional
  m-estimation.
\newblock \emph{Advances in neural information processing systems},
  27:\penalty0 685--693, 2014.

\bibitem[Kakade et~al.(2011)Kakade, Kanade, Shamir, and
  Kalai]{kakade2011efficient}
Sham~M Kakade, Varun Kanade, Ohad Shamir, and Adam Kalai.
\newblock Efficient learning of generalized linear and single index models with
  isotonic regression.
\newblock In \emph{Advances in Neural Information Processing Systems}, pages
  927--935, 2011.

\bibitem[Kalai and Sastry(2009)]{kalai2009isotron}
Adam~Tauman Kalai and Ravi Sastry.
\newblock The isotron algorithm: High-dimensional isotonic regression.
\newblock In \emph{COLT}. Citeseer, 2009.

\bibitem[Li and Genton(2009)]{li2009single}
Yehua Li and Marc~G Genton.
\newblock Single-index additive vector autoregressive time series models.
\newblock \emph{Scandinavian Journal of Statistics}, 36\penalty0 (3):\penalty0
  369--388, 2009.

\bibitem[Liu and Barber(2018)]{liu2018hard}
Haoyang Liu and Rina~Foygel Barber.
\newblock Between hard and soft thresholding: optimal iterative thresholding
  algorithms, 2018.

\bibitem[L{\"u}tkepohl(2013)]{lutkepohl2013vector}
Helmut L{\"u}tkepohl.
\newblock Vector autoregressive models.
\newblock In \emph{Handbook of Research Methods and Applications in Empirical
  Macroeconomics}. Edward Elgar Publishing, 2013.

\bibitem[Mair et~al.(2009)Mair, Hornik, and de~Leeuw]{mair2009isotone}
Patrick Mair, Kurt Hornik, and Jan de~Leeuw.
\newblock Isotone optimization in r: pool-adjacent-violators algorithm (pava)
  and active set methods.
\newblock \emph{Journal of statistical software}, 32\penalty0 (5):\penalty0
  1--24, 2009.

\bibitem[Mark et~al.(2019{\natexlab{a}})Mark, Raskutti, and Willett]{Mark_2019}
Benjamin Mark, Garvesh Raskutti, and Rebecca Willett.
\newblock Network estimation from point process data.
\newblock \emph{IEEE Transactions on Information Theory}, 65\penalty0
  (5):\penalty0 2953–2975, May 2019{\natexlab{a}}.
\newblock ISSN 1557-9654.
\newblock \doi{10.1109/tit.2018.2875766}.
\newblock URL \url{http://dx.doi.org/10.1109/tit.2018.2875766}.

\bibitem[Mark et~al.(2019{\natexlab{b}})Mark, Raskutti, and
  Willett]{mark2019estimating}
Benjamin Mark, Garvesh Raskutti, and Rebecca Willett.
\newblock Estimating network structure from incomplete event data.
\newblock In \emph{The 22nd International Conference on Artificial Intelligence
  and Statistics}, pages 2535--2544. PMLR, 2019{\natexlab{b}}.

\bibitem[Naik and Tsai(2001)]{naik2001single}
Prasad~A Naik and Chih-Ling Tsai.
\newblock Single-index model selections.
\newblock \emph{Biometrika}, 88\penalty0 (3):\penalty0 821--832, 2001.

\bibitem[Neykov(2019)]{neykov2019}
Matey Neykov.
\newblock Isotonic regression meets lasso.
\newblock \emph{Electron. J. Statist.}, 13\penalty0 (1):\penalty0 710--746,
  2019.
\newblock \doi{10.1214/19-EJS1537}.
\newblock URL \url{https://doi.org/10.1214/19-EJS1537}.

\bibitem[Pandit et~al.(2019)Pandit, Sahraee-Ardakan, Amini, Rangan, and
  Fletcher]{p2019highdimensional}
Parthe Pandit, Mojtaba Sahraee-Ardakan, Arash~A. Amini, Sundeep Rangan, and
  Alyson~K. Fletcher.
\newblock High-dimensional bernoulli autoregressive process with long-range
  dependence, 2019.

\bibitem[Raginsky et~al.(2012)Raginsky, Willett, Horn, Silva, and
  Marcia]{Raginsky_2012}
Maxim Raginsky, Rebecca~M. Willett, Corinne Horn, Jorge Silva, and Roummel~F.
  Marcia.
\newblock Sequential anomaly detection in the presence of noise and limited
  feedback.
\newblock \emph{IEEE Transactions on Information Theory}, 58\penalty0
  (8):\penalty0 5544–5562, Aug 2012.
\newblock ISSN 1557-9654.
\newblock \doi{10.1109/tit.2012.2201375}.
\newblock URL \url{http://dx.doi.org/10.1109/TIT.2012.2201375}.

\bibitem[Richey(2008)]{richey2008autoregressive}
Sean Richey.
\newblock The autoregressive influence of social network political knowledge on
  voting behaviour.
\newblock \emph{British Journal of Political Science}, pages 527--542, 2008.

\bibitem[Rohe et~al.(2011)Rohe, Chatterjee, and Yu]{Rohe_2011}
Karl Rohe, Sourav Chatterjee, and Bin Yu.
\newblock Spectral clustering and the high-dimensional stochastic blockmodel.
\newblock \emph{The Annals of Statistics}, 39\penalty0 (4):\penalty0
  1878–1915, Aug 2011.
\newblock ISSN 0090-5364.
\newblock \doi{10.1214/11-aos887}.
\newblock URL \url{http://dx.doi.org/10.1214/11-AOS887}.

\bibitem[Scaillet(2004)]{scaillet2004nonparametric}
Olivier Scaillet.
\newblock Nonparametric estimation and sensitivity analysis of expected
  shortfall.
\newblock \emph{Mathematical Finance: An International Journal of Mathematics,
  Statistics and Financial Economics}, 14\penalty0 (1):\penalty0 115--129,
  2004.

\bibitem[Shephard et~al.(1995)]{shephard1995generalized}
Neil Shephard et~al.
\newblock Generalized linear autoregressions.
\newblock Technical report, Economics Group, Nuffield College, University of
  Oxford, 1995.

\bibitem[Smith and Brown(2003)]{smith2003estimating}
Anne~C Smith and Emery~N Brown.
\newblock Estimating a state-space model from point process observations.
\newblock \emph{Neural computation}, 15\penalty0 (5):\penalty0 965--991, 2003.

\bibitem[Stomakhin et~al.(2011{\natexlab{a}})Stomakhin, Short, and
  Bertozzi]{Stomakhin_2011}
Alexey Stomakhin, Martin~B Short, and Andrea~L Bertozzi.
\newblock Reconstruction of missing data in social networks based on temporal
  patterns of interactions.
\newblock \emph{Inverse Problems}, 27\penalty0 (11):\penalty0 115013, oct
  2011{\natexlab{a}}.
\newblock \doi{10.1088/0266-5611/27/11/115013}.
\newblock URL \url{https://doi.org/10.1088%2F0266-5611%2F27%2F11%2F115013}.

\bibitem[Stomakhin et~al.(2011{\natexlab{b}})Stomakhin, Short, and
  Bertozzi]{stomakhin2011reconstruction}
Alexey Stomakhin, Martin~B Short, and Andrea~L Bertozzi.
\newblock Reconstruction of missing data in social networks based on temporal
  patterns of interactions.
\newblock \emph{Inverse Problems}, 27\penalty0 (11):\penalty0 115013,
  2011{\natexlab{b}}.

\bibitem[Wang et~al.(2010)Wang, Xue, Zhu, Chong, et~al.]{wang2010estimation}
Jane-Ling Wang, Liugen Xue, Lixing Zhu, Yun~Sam Chong, et~al.
\newblock Estimation for a partial-linear single-index model.
\newblock \emph{The Annals of statistics}, 38\penalty0 (1):\penalty0 246--274,
  2010.

\bibitem[Wang et~al.(2016)Wang, Xie, Du, and Song]{wang2016isotonic}
Yichen Wang, Bo~Xie, Nan Du, and Le~Song.
\newblock Isotonic hawkes processes.
\newblock In \emph{International conference on machine learning}, pages
  2226--2234, 2016.

\bibitem[Wu et~al.(2011)Wu, Lin, and Yu]{TracyZWu2011}
Tracy~Z. Wu, Haiqun Lin, and Yan Yu.
\newblock Single-index coefficient models for nonlinear time series.
\newblock \emph{Journal of Nonparametric Statistics}, 23\penalty0 (1):\penalty0
  37--58, 2011.
\newblock \doi{10.1080/10485252.2010.497554}.
\newblock URL \url{https://doi.org/10.1080/10485252.2010.497554}.

\bibitem[Xue and Zhu(2006)]{xue2006empirical}
Liu-Gen Xue and Lixing Zhu.
\newblock Empirical likelihood for single-index models.
\newblock \emph{Journal of Multivariate Analysis}, 97\penalty0 (6):\penalty0
  1295--1312, 2006.

\bibitem[Yang and Barber(2017)]{yang2017contraction}
Fan Yang and Rina~Foygel Barber.
\newblock Contraction and uniform convergence of isotonic regression, 2017.

\bibitem[Zhang et~al.(2002)]{zhang2002risk}
Cun-Hui Zhang et~al.
\newblock Risk bounds in isotonic regression.
\newblock \emph{The Annals of Statistics}, 30\penalty0 (2):\penalty0 528--555,
  2002.

\bibitem[Zhou and Raskutti(2018)]{zhou2018non}
Hao~Henry Zhou and Garvesh Raskutti.
\newblock Non-parametric sparse additive auto-regressive network models.
\newblock \emph{IEEE Transactions on Information Theory}, 65\penalty0
  (3):\penalty0 1473--1492, 2018.

\bibitem[Zhou et~al.(2013)Zhou, Zha, and Song]{zhou2013learning}
Ke~Zhou, Hongyuan Zha, and Le~Song.
\newblock Learning social infectivity in sparse low-rank networks using
  multi-dimensional hawkes processes.
\newblock In \emph{Artificial Intelligence and Statistics}, pages 641--649.
  PMLR, 2013.

\bibitem[Zhu and Wang(2011)]{zhu2011estimation}
Fukang Zhu and Dehui Wang.
\newblock Estimation and testing for a poisson autoregressive model.
\newblock \emph{Metrika}, 73\penalty0 (2):\penalty0 211--230, 2011.

\end{thebibliography}

\section{Appendix}
\label{section: appendix}

\subsection{Proof of \cref{lemma: iterative bound}}
\begin{proof}
    Using three point identity, we have
    \begin{equation}
        \begin{split}
           & \parallel \Phi_{s_j}(\tilde u_j^{(k)})- u_j^{*}\parallel_2^2 \\
           = & \parallel u_j^{(k-1)} - u_j^{*} \parallel_2^2 - \parallel u_j^{(k-1)} -\Phi_{s_j}(\tilde u_j^{(k)}) \parallel_2^2 +
            2 \langle u_j^{(k-1)} - \Phi_{s_j}(\tilde u_j^{(k)}), u_j^{*} - \Phi_{s_j}(\tilde u_j^{(k)})\rangle;
        \end{split}
        \label{eq: 3p_id}
    \end{equation}
    First we bound the third term in \cref{eq: 3p_id}, through
    \begin{equation}
        \begin{split}
            &\left\langle u_j^{(k-1)} - \Phi_{s_j}(\tilde u_j^{(k)}), u_j^{*} - \Phi_{s_j}(\tilde u_j^{(k)})\right\rangle \\
            = &\left\langle u_j^{(k-1)} - \tilde u_j^{(k)}, u_j^{*} - \Phi_{s_j}(\tilde u_j^{(k)})\right\rangle + \left\langle \tilde u_j^{(k)} - \Phi_{s_j}(\tilde u_j^{(k)}),u_j^{*} - \Phi_{s_j}(\tilde u_j^{(k)})\right\rangle\\
            \leq &  \frac{\sqrt{s_j^*}}{2\sqrt{s_j}} \parallel u_j^* -\Phi_{s_j}(\tilde u_j^{(k)}) \parallel_2^2,\\
            &+\left\langle \eta_j \mathcal{P}_{u_j^{(k-1)}}^{\perp}
            \left[\frac{1}{T} \X_{-T}^T(\X_{-0,j} - \iso_{\bfv_{k-1}}(\X_{-0,j}))\right] , \Phi_{s_j}(\tilde u_j^{(k)}) - u_j^* \right\rangle\\
        \end{split}
    \end{equation}
    where the last step comes from the \cref{lemma: hard_thres}.
    
    Hence we have 
    \begin{equation}
        \begin{split}
            & (1-\frac{\sqrt{s_j^*}}{\sqrt s_j}) \parallel u_j^* - \Phi_{s_j}(\tilde u_j^{(k)}) \parallel_2^2 \\
            \leq &\parallel u_j^{(k-1)} - u_j^{*} \parallel_2^2 - \parallel u_j^{(k-1)} -\Phi_{s_j}(\tilde u_j^{(k)}) \parallel_2^2 \\
            &+2 \frac{\eta_j}{T} \langle \mathcal{P}_{u_j^{(k-1)}}^{\perp}
            [\X_{-T}^T(\X_{-0,j} - \iso_{\bfv_{k-1}}(\X_{-0,j}))] , \Phi_{s_j}(\tilde u_j^{(k)}) - u_j^* \rangle.
        \end{split}
        \label{eq: prf1}
    \end{equation}
    Splitting the inner product term in the way that 
    \begin{equation}
        \begin{split}
            &\X_{-0,j} - \iso_{\bfv_{k-1}}(\X_{-0,j}) \\
            = & \X_{-0,j} - f_j^*(\X_Tu_j^*)
            +f_j^*(\X_Tu_j^*) - \iso_{\bfv_{k-1}}(f_j^*(\X_{-T}u_j^*) )\\
             &+ \iso_{\bfv_{k-1}}(f_j^*(\X_{-T}u_j^*) )- \iso_{\bfv_{k-1}}(\X_{-0,j});\\
        \end{split}
        \label{eq: splitting}
    \end{equation}

    So as to ease notation, since we are dealing with the $j$-th variate, which shares exactly the same method and theoretical technique across all $j\in [M]$, we drop the subscript $j$ in the following of this proof. Specifically,
    we denote the projection operator $\mathcal{P}_{u_j^{(k-1)}}^{\perp}$ as $\mathcal{P}_{k-1}^{\perp}$; $u_j^*$ as $u^*$ and $u_j^{(k)}$ as $u^{(k)}$; $\X_{-0,j}$ as $X_{-0}$; $s_j^*$ and $s_j$ as $s^*$ and $s$ respectively, $f_j^*$ as $f^*$.

    By the property of inner product, we have
    \begin{equation}
        \begin{split}
            &\left\langle \mathcal{P}_{k-1}^{\perp}
            \left(
                \X_{-T}^T(
                X_{-0} - f^*(\X_{-T}u^*) 
                ) 
            \right), 
            \Phi_s(\tilde u^{(k)}) - u^* 
            \right\rangle\\
            \leq& \parallel 
            \X_{-T}^T(X_{-0} - f^*(\X_{-T}u^*))
            \parallel_{\infty} 
            \cdot 
            \parallel 
            \mathcal{P}_{k-1}^{\perp} 
            (\Phi_s(\tilde u^{(k)}) - u^*)
            \parallel_1\\
            \leq & \sqrt{2s +s^*}\parallel 
            \X_{-T}^T(X_{-0} - f^*(\X_{-T}u^*))
            \parallel_{\infty} 
            \cdot
            \parallel \Phi_s(\tilde u^{(k)}) - u^*\parallel_2,
        \end{split}
        \label{eq: I}
    \end{equation}
    where the last step comes from the inequality of norms that $\parallel x\parallel_1 \leq \sqrt{\parallel x \parallel_0}\cdot \parallel x \parallel_2$ for any vector $x$, and the fact that $\mathcal{P}_{k-1}^{\perp} (\Phi_s(\tilde u^{(k)}) - u^*)$ has at most $2s+s^*$ non-zero elements since $\parallel u^{(k-1)}\parallel_0 = s, \parallel u^*\parallel_0 = s^*$ and $\parallel \Phi_s(\tilde u^{(k)})\parallel_0 = s$.

    By the assumption of data boundedness in \cref{asspt: data_bound} and Cauchy inequality, we have
    \begin{equation}
        \begin{split}
            &\left\langle \mathcal{P}_{k-1}^{\perp}
            \left(
                \X_{-T}^T(
                \iso_{\bfv_{k-1}}(f^*(\X_{-T}u^*)) - \iso_{\bfv_{k-1}}(X_{-0}) ) 
            \right), 
            \Phi_s(\tilde u^{(k)}) - u^* 
            \right\rangle\\
            \leq& \parallel 
            \iso_{\bfv_{k-1}}(f^*(\X_{-T}u^*)) - \iso_{\bfv_{k-1}}(X_{-0})
            \parallel_{2} 
            \cdot 
            \parallel 
            \X_{-T} \mathcal{P}_{k-1}^{\perp} 
            (\Phi_s(\tilde u^{(k)}) - u^*)
            \parallel_2\\
            \leq & 
            \parallel 
            \iso_{\bfv_{k-1}}(f^*(\X_{-T}u^*)) - \iso_{\bfv_{k-1}}(X_{-0})
            \parallel_{2} 
            \cdot \sqrt{\beta T}
            \parallel \Phi_s(\tilde u^{(k)}) - u^*\parallel_2.
        \end{split}\label{eq: III}
    \end{equation}

    To deal with the remaining part of the inner product generated from the last splitting term, first we use the relationship between inner product and dual norms for 2-norm again to have 
    \begin{equation}
        \begin{split}
            &\left\langle
            f^*(\X_{-T}u^*)  - \iso_{\bfv_{k-1}}f^*(\X_{-T}u^*),
            \X_{-T} \mathcal{P}_{k-1}^{\perp} \Phi_s(\tilde u^{(k)})
            \right\rangle\\
            \leq & \parallel f^*(\X_{-T}u^*)  - \iso_{\bfv_{k-1}}f^*(\X_{-T}u^*) \parallel_2 
            \parallel \X_{-T} \mathcal{P}_{k-1}^{\perp} \Phi_s(\tilde u^{(k)}) \parallel_2\\
            \leq & \parallel f^*(\X_{-T}u^*)  - \iso_{\bfv_{k-1}}f^*(\X_{-T}u^*) \parallel_2  \cdot \sqrt{\beta T} \parallel \mathcal{P}_{k-1}^{\perp} \Phi_s(\tilde u^{(k)}) \parallel_2\\
            \leq& \parallel f^*(\X_{-T}u^*) - \iso_{\bfv_{k-1}}f^*(\X_{-T}u^*) \parallel_2 \cdot \sqrt{\beta T} \parallel \Phi_s(\tilde u^{(k)})- 
            u^{(k-1)} \parallel_2\\
            \leq & \frac{1}{2L_j} \parallel f^*(\X_{-T}u^*) - \iso_{\bfv_{k-1}}f^*(\X_{-T}u^*) \parallel_2^2 +
             \frac{L_j\beta T}{2} \parallel \Phi_s(\tilde u^{(k)}) - u^{(k-1)} \parallel^2_2,\\
        \end{split}
        \label{eq: IV1}
    \end{equation}
    where the third line is true again due to \cref{asspt: data_bound}; the last line comes from the fact that for any $a,b \in \bbR$ and any $L>0$, we have $ab \leq \frac{a^2}{2L}+ \frac{b^2L}{2}$.

    On the other hand, 
    \begin{equation}
        \begin{split}
            &\left\langle
            f^*(\X_{-T}u^*)  - \iso_{\bfv_{k-1}}f^*(\X_{-T}u^*),
            \X_{-T} \mathcal{P}_{k-1}^{\perp} (u^*)
            \right\rangle\\
            =& \left\langle
            f^*(\X_{-T}u^*)  - \iso_{\bfv_{k-1}}f^*(\X_{-T}u^*),
            \X_{-T}\left(
                u^* - 
                \langle u^* ,u^{(k-1)} \rangle u^{(k-1)}
            \right)
            \right\rangle\\
            = & \left\langle
            f^*(\X_{-T}u^*)  - \iso_{\bfv_{k-1}}f^*(\X_{-T}u^*), \X_{-T}u^*
            \right\rangle\\
            &-
            \langle u^*, u^{(k-1)}\rangle
            \left\langle
            f^*(\X_{-T}u^*)  - \iso_{\bfv_{k-1}}f^*(\X_{-T}u^*) ,
            \X_{-T} u^{(k-1)}
            \right\rangle\\
            \geq& \frac{1}{L_j}\parallel f^*(\X_{-T}u^*)  - \iso_{\bfv_{k-1}}f^*(\X_{-T}u^*) \parallel_2^2\\
            &- \langle u^*, u^{(k-1)}\rangle
            \left\langle
            f^*(\X_{-T}u^*)  - \iso_{\bfv_{k-1}}f^*(\X_{-T}u^*) ,
            \bfv_{k-1}
            \right\rangle\\
            \geq & \frac{1}{L_j}\parallel f^*(\X_{-T}u^*)  - \iso_{\bfv_{k-1}}f^*(\X_{-T}u^*) \parallel_2^2,\\
        \end{split}
        \label{eq: IV2}
    \end{equation}
    where the forth line is derived from \cref{lemma: lip_bound};
    the last step comes from the acuteness of the angle between $u^*$ and $u^{(k-1)}$ as assumed, and the fact that 
    \begin{equation}
        \langle \omega - \iso_v (\omega) , v \rangle \leq 0
    \end{equation}
    for any vectors $v,\omega \in \bbR^T$, since $\iso_v(\omega)$ is the projection of $\omega$ onto the convex cone satisfying the isotonic constraints, while $v$ itself is a vector contained in this convex cone.

    Combining \cref{eq: IV1} and \cref{eq: IV2}, we have 
    \begin{equation}
        \begin{split}
            &\left\langle
              \mathcal{P}_{k-1}^{\perp}
            \left[
                \X_{-T}
            \left(
                f^*(\X_{-T}u^*)  - \iso_{\bfv_{k-1}}f^*(\X_{-T}u^*)
            \right) 
            \right], 
            \Phi_s(\tilde u^{(k)}) - u^* 
            \right\rangle\\
            = & \left\langle
            f^*(\X_{-T}u^*)  - \iso_{\bfv_{k-1}}f^*(\X_{-T}u^*),
            \X_{-T} \mathcal{P}_{k-1}^{\perp} (\Phi_s(\tilde u^{(k)}))
            \right\rangle \\
            &-
            \left\langle
            f^*(\X_{-T}u^*)  - \iso_{\bfv_{k-1}}f^*(\X_{-T}u^*),
            \X_{-T} \mathcal{P}_{k-1}^{\perp} (u^*)
            \right\rangle\\
            \leq & \frac{L_j\beta T}{2} \parallel \Phi_s(\tilde u^{(k)}) - u^{(k-1)} \parallel^2_2 -
            \frac{1}{2L_j}\parallel f^*(\X_{-T}u^*)  - \iso_{\bfv_{k-1}}f^*(\X_{-T}u^*) \parallel_2^2.
        \end{split}
        \label{eq: IV}
    \end{equation}

    Plugging \cref{eq: I}, \cref{eq: III} and \cref{eq: IV} into \cref{eq: prf1}, we have
    \begin{equation}
        \begin{split}
            &(1-\sqrt{\frac{s^*}{s}})\parallel u^* - 
            \Phi_s(\tilde u^{(k)}) \parallel_2^2\\
            \leq & 
            (\eta_j L_j\beta -1) \parallel 
             \Phi_s(\tilde u^{(k)}) - u^{(k-1)}\parallel_2^2 +
            2\eta_j C(Z,u_j^{(k-1)})
            \parallel u^* -\Phi_s(\tilde u^{(k)})\parallel_2\\
            &+ \parallel u^* - u^{(k-1)}\parallel_2^2
            - \frac{\eta_j}{TL_j} \parallel 
            f^*(\X_{-T}u^*)  - \iso_{\bfv_{k-1}}f^*(\X_{-T}u^*)
            \parallel_2^2,\\
        \end{split}
    \end{equation}
    
    Let the step size be chosen as $\eta_j = \frac{1}{L_j\beta}$, by Cauchy inequality, for any $\delta_j >0$,
    \begin{equation}
        \begin{split}
            &\frac{2}{L_j\beta} C(Z,u_j^{(k-1)})
        \parallel u^* -\Phi_s(\tilde u^{(k)})\parallel_2 \\
        \leq& \frac{\delta_j \alpha_j}{L_j^2 \beta}
        \parallel u^* -\Phi_s(\tilde u^{(k)})\parallel_2^2 +
        \frac{1}{\beta \alpha_j\delta_j}C^2(Z,u_j^{(k-1)}).
        \end{split}
    \end{equation}

    Hence, putting back the subscript $j$ we have that 
    \begin{equation}
        \begin{split}
        &(1 - \sqrt{\frac{s_j^*}{s_j}} - \frac{\delta_j \alpha_j}{L_j^2 \beta}) \parallel u_j^* - 
        \Phi_{s_j}(\tilde u_j^{(k)}) \parallel_2^2\\
        \leq&\parallel u_j^* - u_j^{(k-1)}\parallel_2^2 + \frac{1}{\beta \alpha_j\delta_j}C^2(Z,u_j^{(k-1)})
        - \frac{\eta_j}{L_j} \left(
            \alpha_j\parallel u_j^* - u_j^{(k-1)} \parallel_2^2 -\epsilon_j^2
        \right)\\
        = &
        (1-\frac{\alpha_j}{L_j^2\beta})\parallel u_j^* - u_j^{(k-1)}\parallel_2^2
        + \frac{C^2(Z,u_j^{(k-1)})}{\beta \alpha_j\delta_j}+
        \frac{\epsilon_j^2}{L_j^2\beta}.
        \end{split}
        \label{eq: almost}
    \end{equation}

    Next we only need to show that $\parallel u_j^* - 
    u_j^{(k)} \parallel_2^2 
    \leq \parallel u_j^* - 
    \Phi_{s_j}(\tilde u_j^{(k)}) \parallel_2^2 $ to complete the proof. 
    Let $S_j^{(k-1)}$ be the support of $u_j^{(k-1)}$, then we know $|S_j^{(k-1)}| = s_j$.
    Since $\Phi_{s_j}(\cdot)$ is the hard-thresholding operator that finds the largest $s_j$ elements in absolute values, we have
    \begin{equation}
        \begin{split}
            \left\|\Phi_{s_j}(\tilde u_j^{(k)})\right\|_2^2 &\geq 
            \left\| \left(\tilde u_j^{(k)} \right)_{S_j^{(k-1)}} \right\|_2^2\\
            & = \left\| \left(
                u_j^{(k-1)} + \eta_j \mathcal{P}_{u_j^{(k-1)}}^{\perp}
                \left[\frac{1}{T} \X_{-T}^T(\X_{-0,j} - \iso_{\bfv_{k-1}}(\X_{-0,j}))\right]
             \right)_{S_j^{(k-1)}} \right\|_2^2\\
             &= \left\| 
                u_j^{(k-1)} + \eta_j \left(\mathcal{P}_{u_j^{(k-1)}}^{\perp}
                \left[\frac{1}{T} \X_{-T}^T(\X_{-0,j} - \iso_{\bfv_{k-1}}(\X_{-0,j}))\right]
             \right)_{S_j^{(k-1)}} \right\|_2^2\\
             & \geq \| u_j^{(k-1)}\|_2^2=1.
        \end{split}
    \end{equation}

    Therefore, by the fact that $u_j^{(k)} = \frac{\Phi_{s_j}(\tilde u_j^{(k)})}{\| \Phi_{s_j}(\tilde u_j^{(k)}) \|_2}$, we know that $u_j^{(k)}$ is the projection of $\Phi_{s_j}(\tilde u_j^{(k)}) \geq 1$ not only onto the unit-sphere, by also onto the unit-ball. 
    Since a unit ball $\{u\in \bbR^M: \|u\|_2^2 \leq 1\}$ is a convex set containing $u_j^*$, we have 
    \begin{equation}
        \langle \Phi_{s_j}(\tilde u_j^{(k)}) - u_j^{(k)}, u_j^* - u_j^{k}\rangle \leq 0;
    \end{equation}
    
    Equipped with the three-point identity, we have
    \begin{equation}
        \| u_j^* - \Phi_{s_j}(\tilde u_j^{(k)}) \|_2^2
        = 
        \| u_j^* - u_j^{(k)}\|_2^2 + \|\Phi_{s_j}(\tilde u_j^{(k)}) - u_j^{(k)}\|_2^2 + 2 \rangle u_j^{(k)} - \Phi_{s_j}(\tilde u_j^{(k)}), u_j^* - u_j^{(k)}\rangle \geq \| u_j^* - u_j^{(k)}\|_2^2.
        \label{eq: proj}
    \end{equation}

    Combining \cref{eq: almost} and \cref{eq: proj}, we finish the proof of \cref{lemma: iterative bound}.
\end{proof}

\subsection{Proof of \cref{lemma: infty_norm_z}}
\begin{proof}
    For any $j\in [M]$ of interest, the following inequalities hold by union bound for any $y>0$: 
    \begin{equation}
        \begin{split}
            &P\left(
                \frac{1}{T}\parallel \X_{-T}^T \left(\X_{-0,j} - f_j^{*}(\X_{-T}\cdot u_j^{*})\right) \parallel_{\infty} \geq y
            \right)\\
            \leq & \sum_{i=1}^M P\left(
                \frac{1}{T}\left|
                \sum_{t=0}^{T-1}
                X_{t,i}\left(X_{t+1,j} - f_j^{*}(X_{t}^T\cdot u_j^{*})\right)
                \right|\geq y
            \right)\\
            = &\sum_{i=1}^M P\left(
                \frac{1}{T}\left|\sum_{t=0}^{T-1}  X_{t,i}
                Z_{t+1,j}
                \right|\geq y
            \right).
            \label{eq: ub_init}
        \end{split}
    \end{equation}

    By the assumption that for any $\lambda \in \bbR$, the noises $Z_{t,j}, t = 1,\dots, T$ are conditionally sub-Gaussian in the way that $\bbE\left[\exp\left( \lambda Z_{t,j}\right) | \mathcal{F}_{t-1}\right] \leq e^{\sigma_j^2\lambda^2/2}$,
    and the fact that $\X$ is entrywise bounded by $M_x$, we have
    \begin{equation}
        \begin{split}
            &\bbE\left(
                    e^{\lambda X_{t-1,i}Z_{t,j}}
                    |\mathcal{F}_{t-1}
                \right)\\
            \leq& \bbE \left[
                \exp \left(
                \frac{\lambda^2  \sigma_j^2X_{t-1,i}^2}{2}  
            \right) |\mathcal{F}_{t-1}
            \right]\\
            \leq & \exp \left(
                \frac{\lambda^2  \sigma_j^2 M_x^2}{2}  
            \right), ~ t = 1, \dots, T.
        \end{split}
    \end{equation}
    We could then bound the moment generating function of the sum $\sum_{t=1}^{T} X_{t-1,i}
    Z_{t,j}$ in the following recursive manner using conditional expectations. Define $S_{n,i,j} = \sum_{t=1}^{n} X_{t-1,i}
    Z_{t,j}, ~ n\in [T]$. For any $\lambda \in \bbR$ and any $i =1,\dots, M$,
    \begin{equation}
        \begin{split}
            \bbE\left[\exp (\lambda S_{n,i,j})\right]= &\bbE\left[
                \exp\left(
                    \lambda \sum_{t= 1}^{n} X_{t-1,i}Z_{t,j}
                \right)
            \right]\\
            = & \bbE \left[
                \bbE\left(
                    e^{\lambda \sum_{t= 1}^{n} X_{t-1,i}Z_{t,j}}
                    |\mathcal{F}_{n-1}
                \right)
            \right]\\
            = & \bbE\left[
                \exp \left(\lambda \sum_{t= 1}^{n-1} X_{t-1,i}Z_{t,j}\right)\cdot 
                \bbE\left(
                    e^{\lambda X_{n-1,i}Z_{n,j}}
                    |\mathcal{F}_{n-1}
                \right)
            \right]\\
            \leq& \exp \left(
                \frac{\lambda^2  \sigma_j^2 M_x^2}{2}  
            \right) \cdot \bbE\left[
                \exp \left(\lambda \sum_{t= 1}^{n-1} X_{t-1,i}Z_{t,j}\right)\right]\\
            = & \exp \left(
                \frac{\lambda^2  \sigma_j^2 M_x^2}{2}  
            \right) \bbE\left[\exp (\lambda S_{n-1,i,j})\right].
        \end{split}
    \end{equation}
    Therefore we have 
    \begin{equation}
        \begin{split}
            \bbE\left[\exp (\lambda S_{T,i,j})\right] &\leq
            \exp \left(
                \frac{ (T-1)\lambda^2  \sigma_j^2 M_x^2}{2}  
            \right)\cdot \bbE\left[\exp (\lambda S_{1,i,j})\right]\\
            &= \exp \left(
                \frac{ (T-1)\lambda^2  \sigma_j^2 M_x^2}{2}  
            \right)\cdot 
            \bbE\left[\left(
                \exp (\lambda X_{0,i} Z_{1,j})
            \right) | \mathcal{F}_{0}
            \right]\\
            &\leq \exp \left(
                \frac{ T\lambda^2  \sigma_j^2 M_x^2}{2}  
            \right).
        \end{split}
    \end{equation}
    By the fact that $\bbE\left(
        Z_{t,j}|\mathcal{F}_{t-1}
    \right) = 0$, we have
    \begin{equation}
        \bbE\left[
            S_{T,i,j}
        \right] = \sum_{t=1}^T \bbE \left[
            X_{t-1,i} Z_{t,j}
        \right] = \sum_{t=1}^T \bbE \left[
            \bbE\left(
                X_{t-1,i} Z_{t,j} | \mathcal{F}_{t-1}
            \right)
        \right] = \sum_{t=1}^T \bbE \left[X_{t-1,i}
            \bbE\left(
                 Z_{t,j} | \mathcal{F}_{t-1}
            \right)
        \right] = 0.
    \end{equation}

    Hence, by definition, $S_{T,i,j}$ is sub-Gaussian with parameter $\sqrt{T} \sigma_j M_x$ and mean zero. Using the concentration inequality and \cref{eq: ub_init}, we have
    \begin{equation}
        \begin{split}
            &P\left(
                \frac{1}{T}\parallel \X_{-T}^T \left(\X_{-0,j} - f_j^{*}(\X_{-T}\cdot u_j^{*})\right) \parallel_{\infty} \geq y
            \right)\\
             \leq& \sum_{i=1}^M P\left(
                \frac{1}{T} |S_{T,i,j}| \geq y
            \right)\\
            \leq& M\exp\left(-
                \frac{T y^2}{2 \sigma_j^2M_x^2}
            \right).
        \end{split}
    \end{equation}
\end{proof}

\subsection{Proof of \cref{lemma: sup_2_norm}}
Before showing the uniform $\ell_2$ norm bound in \cref{lemma: sup_2_norm} over $\bfV= \{ \bfv = \X_{-T} \bfu:  \bfu\in \mathcal{S}^{M-1} \text{with sparsity } s_j\}$ , we first need the following lemma for the $\ell_2$ norm bound for arbitrary vector $\bfv \in \bfV$.

\begin{lemma}\label{lemma: 2norm bound}
    For any $\bfv \in \bfV= \{ \bfv = \X_{-T} \bfu:  \bfu\in \mathcal{S}^{M-1} \text{with sparsity } s_j\}$, with probability at least $1-2\gamma$, 
    \begin{equation}
        \begin{split}
            & \parallel \iso_{\bfv} (f_j^{*}(\X_{-T}u_j^{*})) - \iso_{\bfv}(\X_{-0,j}) \parallel_2^2\\
            \leq &16 T^{\frac{1}{3}} \left[\left(2\sqrt{2\sigma_j^2 \log \frac{2T}{\gamma} } + 2 L_j M_x \sqrt{s_j^{*}} \right)
                \sigma_j^2 \log( \frac{T(T+1)}{\gamma})
            \right]^{\frac{2}{3}} + 4 \sigma_j^2 \log(\frac{T(T+1)}{\gamma}).
        \end{split}
    \end{equation}
\end{lemma}

\begin{proof}
    (\cref{lemma: 2norm bound})

    For any $\bfv \in \bfV$, we first bound the range of $\iso_{\bfv} (f_j^{*}(\X_{-T}u_j^{*}))$ by the nature of isotonic regression as well as the the monotonicity and the $L_j$-Lipschitz continuity of the function $f_j^{*}$:
    \begin{equation}
        \begin{split}
            &\left|
            (\iso_{\bfv} (f_j^{*}(\X_{-T}u_j^{*})))_{(n)} - 
            (\iso_{\bfv} (f_j^{*}(\X_{-T}u_j^{*})))_{(1)}
            \right|\\
            \leq & \left|
            (f_j^{*}(\X_{-T}u_j^{*}))_{max} - (f_j^{*}(\X_{-T}u_j^{*}))_{min}
            \right|\\
            \leq &\left|
            f_j^{*}(M_x \sqrt{s_j^{*}}) - f_j^{*}( - M_x \sqrt{s_j^{*}})
            \right|\\
            \leq &2 L_j M_x \sqrt{s_j^{*}},
        \end{split}
    \end{equation}
    where we made use of the boundedness of $\X_{-T}$ and sparsity of $u_j^{*}$ in the way that 
    \begin{equation}
       \begin{split}
            \parallel\X_{-T}u_j^{*}\parallel_{\infty} &= \max_{t\in [T]-1}(|\langle X_t, u_j^{*}\rangle|) \\
            &\leq 
        \max_{t\in [T]-1} \parallel X_t\parallel_{\infty} \parallel u_j^{*}\parallel_1 \leq M_x \sqrt{\parallel u_j^{*}\parallel_{0}} \cdot \parallel u_j^{*} \parallel_2 \leq M_x \sqrt{s_j^{*}}.
       \end{split}
    \end{equation}

    Following that, we further bound the range of $\iso_{\bfv}(\X_{-0,j})$:
    \begin{equation}\label{ineq: range bound2}
        \begin{split}
            &\left| (\iso_{\bfv}(\X_{-0,j}))_{(n)} - (\iso_{\bfv}(\X_{-0,j}))_{(1)} \right|\\
            \leq & \left| (\iso_{\bfv}(\X_{-0,j}))_{(n)} - (\iso_{\bfv}(f_j^{*}(\X_{-T}u_j^{*})))_{(n)} \right|+
            \left| (\iso_{\bfv}(\X_{-0,j}))_{(1)} - (\iso_{\bfv}(f_j^{*}(\X_{-T}u_j^{*})))_{(1)} \right|\\
            &+
            \left| (\iso_{\bfv}(f_j^{*}(\X_{-T}u_j^{*})))_{(n)} - (\iso_{\bfv}(f_j^{*}(\X_{-T}u_j^{*})))_{(1)} \right|\\
            \leq & 2 \parallel \iso_{\bfv}(\X_{-0,j}) - \iso_{\bfv}(f_j^{*}(\X_{-T}u_j^{*})) \parallel_{\infty}\\ 
            &+  \left| (\iso_{\bfv}(f_j^{*}(\X_{-T}u_j^{*})))_{(n)} - (\iso_{\bfv}(f_j^{*}(\X_{-T}u_j^{*})))_{(1)} \right|\\
            \leq & 2 \parallel \X_{-0,j} - f_j^{*}(\X_{-T}u_j^{*}) \parallel_{\infty} +\left| (\iso_{\bfv}(f_j^{*}(\X_{-T}u_j^{*})))_{(n)} - (\iso_{\bfv}(f_j^{*}(\X_{-T}u_j^{*})))_{(1)} \right|\\
            \leq& 2\parallel \Z_j \parallel_{\infty} + 
            2 L_j M_x \sqrt{s_j^{*}}.
        \end{split}
    \end{equation} 
    The first inequality comes from triangle inequality, the second holds directly from the definition of infinity norm, while the third inequality holds because of the contractive property of isotonic regression in \cref{cor: contract}. 

    Let $\mathbfcal{A} = \{\parallel \Z_j \parallel_{\infty} \leq \sqrt{2\sigma_j^2 \log \frac{2T}{\gamma} }\}$. 
    By the sub-Gaussianity of $\Z_j = (Z_{t,j})_{t = 1,\cdots, T}$ with parameter $\sigma_j$ and the union bound inequality, we know that $P(\mathbfcal{A}) \geq 1 - \gamma$.
    Hence, based on \cref{ineq: range bound2}, with probability at least $1-\gamma$ on event $\mathbfcal{A}$,
    \begin{equation}\label{eq: Hj}
        \left| (\iso_{\bfv}(\X_{-0,j}))_{(n)} - (\iso_{\bfv}(\X_{-0,j}))_{(1)} \right| \leq
        \tilde{H_j} :=
        2\sqrt{2\sigma_j^2 \log \frac{2T}{\gamma} } + 2 L_j M_x \sqrt{s_j^{*}}.
    \end{equation}

    Next, we apply a similar proof technique used in (\cite{dai2021convergence}),  cutting the range of $\iso_{\bfv}(\X_{-0,j})$ and $\iso_{\bfv} (f_j^{*}(\X_{-T}u_j^{*}))$ into arbitrarily short segments with $Q$ cutting points $\{M_0 = 1, M_1,\cdots,M_Q = T\}$: for any $n\geq 1$ and fixed $j$,
    on event $\mathbfcal{A}$ we have 
    \begin{equation}\label{eq: range}
        \begin{split}
            & (\iso_{\bfv} (f_j^{*}(\X_{-T}u_j^{*})))_{(M_i)}-
            (\iso_{\bfv} (f_j^{*}(\X_{-T}u_j^{*})))_{(M_{i-1}+1)} \leq \tilde{H_j}/n;\\
            &(\iso_{\bfv}(\X_{-0,j}))_{(M_i)} - (\iso_{\bfv}(\X_{-0,j}))_{(M_{i-1}+1)} \leq \tilde{H_j}/n, \forall i = 1,\cdots,Q.
        \end{split}
    \end{equation}
    Let the index set of $l^{th}$ segment be $\I_l = \{(M_{l-1}+1),(M_{l-1}+2),\cdots,(M_{l})\}$, then we know that \cref{eq: range} is equavalent to:
    \begin{equation}
        \begin{split}
            &\range((\iso_{\bfv} (f_j^{*}(\X_{-T}u_j^{*})))_{\I_l})\leq \tilde{H_j}/n;\\
            &\range((\iso_{\bfv}(\X_{-0,j}))_{\I_l}) \leq \tilde{H_j}/n, \forall l = 1,\cdots,Q.
        \end{split}
    \end{equation}

    Thanks to the above segmentation, we could first bound the distances of points in each segment to its grouped mean:
    \begin{equation}
        \parallel (\iso_{\bfv} (\bfx))_{\I_l} - \overbar{(\iso_{\bfv} (\bfx))_{\I_l}} \mathbf{1}_{|\I_l|} \parallel_2 \leq \sqrt{|\I_l|}\cdot \tilde{H_j}/n, \forall l\in [Q],
    \end{equation}
    which holds for both $\bfx = \X_{-0,j}$ and $\bfx = f_j^{*}(\X_{-T}u_j^{*})$.

    On the other hand, using union bounds and the fact that $\{Z_{t,j}\}_{t=1}^T$ is a martingale difference sequence with conditional sub-Gaussian tail, we could bound the distance of the centers by
    \begin{equation}
        \begin{split}
            &P\left( \max_{l \in [Q]}\{
                \sqrt{|\I_l|}\cdot \left|
    \overbar{(\iso_{\bfv}(\X_{-0,j}))_{\I_l}} -
    \overbar{(\iso_{\bfv} (f_j^{*}(\X_{-T}u_j^{*})))_{\I_l} }\right|
    \} > t
            \right)\\
            \leq&P\left( \max_{l \in [Q]} \sqrt{|\I_l|}\cdot\left|
            \overbar{(\X_{-0,j})_{\I_l}} -
            \overbar{ (f_j^{*}(\X_{-T}u_j^{*}))_{\I_l} }\right| > t
            \right)\\
            \leq& \binom{T+1}{2} P\left(\sqrt{|\I_l|}\cdot
            |\overbar{(\Z_j)_{\I_l}}| >t
            \right)\\
            \leq & T(T+1) \exp\left(- \frac{ t^2}{2 \sigma_j^2}
            \right).
        \end{split}
    \end{equation}
    The second line holds by using contractive property again in \cref{cor: contract} and the third line comes from union bound relaxation. Thus the following event $\mathbfcal{B}$ takes place with probability at least $1 - \gamma$, where
    \begin{equation}
        \mathbfcal{B} = \left\{
             \max_{l \in [Q]}\{
                \sqrt{|\I_l|}\cdot \left|
    \overbar{(\iso_{\bfv}(\X_{-0,j}))_{\I_l}} -
    \overbar{(\iso_{\bfv} (f_j^{*}(\X_{-T}u_j^{*})))_{\I_l} }\right|
    \}
     \leq 
     \sqrt{2\sigma_j^2\log(\frac{T(T+1)}{\gamma})}
    \right\}.
    \end{equation}

    Hence, by triangle inequality, in each segment, we have
    \begin{equation}
        \begin{split}
            &\parallel (\iso_{\bfv} (\X_{-0,j}))_{\I_l} - (\iso_{\bfv}(f_j^{*}(\X_{-T}u_j^{*})))_{\I_l} )\parallel_2\\
            \leq &2 \sqrt{|\I_l|}\tilde{H_j}/n + \sqrt{2\sigma_j^2\log(\frac{T(T+1)}{\gamma})}.
        \end{split} 
    \end{equation}

    Putting all the segments together, we have on the event $\mathbfcal{A} \cap \mathbfcal{B}$,
    \begin{equation}\label{ineq: 2norm}
        \begin{split}
            & \parallel \iso_{\bfv} (\X_{-0,j}) - \iso_{\bfv}(f_j^{*}(\X_{-T}u_j^{*})) )\parallel_2^2\\
            \leq& \sum_{l =1}^Q \left[
                2 \sqrt{|\I_l|}\tilde{H_j}/n + \sqrt{2\sigma_j^2\log(\frac{T(T+1)}{\gamma})}
            \right]^2\\
            \leq & 8 \sum_{l=1}^Q |\I_l| \frac{\tilde{H_j}^2}{n^2} + 4Q \sigma_j^2\log(\frac{T(T+1)}{\gamma})\\
            \leq& \frac{8 T \tilde{H_j}^2}{n^2} + 4 (2n -1) \sigma_j^2\log(\frac{T(T+1)}{\gamma}),
        \end{split}
    \end{equation}
    where the last line comes from the trivial results in the segmentation step that $Q \leq 2n - 1$.

    Take the segmentation parameter $n$ that could roughly minimize the bound in \cref{ineq: 2norm}:
    \begin{equation}
        n = \lceil \left(
        \frac{T \tilde{H_j}^2}{\sigma_j^2 \log(\frac{T(T+1)}{\gamma})}
        \right)^{\frac{1}{3}}
        \rceil,
    \end{equation}
    we have on the event $\mathbfcal{A} \cap \mathbfcal{B}$ ,with probability at least $1- 2\gamma$, 
    \begin{equation}
        \begin{split}
            & \parallel \iso_{\bfv} (\X_{-0,j}) - \iso_{\bfv}(f_j^{*}(\X_{-T}u_j^{*})) )\parallel_2^2\\
            \leq & 16 T^{\frac{1}{3}} \tilde{H_j}^{\frac{2}{3}} \left[
                \sigma_j^2 \log( \frac{T(T+1)}{\gamma})
            \right]^{\frac{2}{3}} + 4 \sigma_j^2 \log(\frac{T(T+1)}{\gamma})\\
            = &16 T^{\frac{1}{3}} \left[\left(2\sqrt{2\sigma_j^2 \log \frac{2T}{\gamma} } + 2 L_j M_x \sqrt{s_j^{*}} \right)
                \sigma_j^2 \log( \frac{T(T+1)}{\gamma})
            \right]^{\frac{2}{3}} \\
            &+ 4 \sigma_j^2 \log(\frac{T(T+1)}{\gamma}),
        \end{split}
        \label{eq: 2-norm almost}
    \end{equation}
    where the last line is derived by plugging in the definition of the range bound $\tilde{H_j}$ in \cref{eq: Hj}.
\end{proof}

With the help of \cref{lemma: 2norm bound}, now we could derive the uniform bound in \cref{lemma: sup_2_norm}.

\begin{proof}
    \emph{(\cref{lemma: sup_2_norm})}
    In \cref{lemma: sup_2_norm}, we intend to uniformly bound
    \begin{equation*}
        \parallel \iso_{\bfv} (\X_{-0,j}) - \iso_{\bfv}(f_j^{*}(\X_{-T}u_j^{*})) )\parallel_2
    \end{equation*}
     over the set $\bfV= \{ \bfv = \X_{-T} \bfu:  \bfu\in \mathcal{S}^{M-1} \text{with sparsity } s_j\}$. To this end, we first notice that for the isotonic regression $\iso_{\bfv}(\cdot)$, it is the ordering of $\bfv$ instead of $\bfv$ itself that makes the difference. 
    That is, for any $\pi$ being a permutation for $\{1,\dots,T\}$, $\bfu,\bfv \in \bbR^T$ are $T$-dimensional vectors satisfying $\bfu_{\pi(1)}\leq \dots \leq \bfu_{\pi(T)}$ and $\bfv_{\pi(1)}\leq \dots \leq \bfv_{\pi(T)}$, then even if $\bfu\neq \bfv$, we still have
    \begin{equation}
        \parallel \iso_{\bfv} (\X_{-0,j}) - \iso_{\bfv}(f_j^{*}(\X_{-T}u_j^{*})) )\parallel_2=
        \parallel \iso_{\bfu} (\X_{-0,j}) - \iso_{\bfu}(f_j^{*}(\X_{-T}u_j^{*})) )\parallel_2.
    \end{equation}

    For any $n$ points $x_1, \dots,x_n \in \bbR^M$ and a given sparsity level $s$, define the set
    \begin{equation}
        \begin{split}
            &\cS_{n,M}^{\text{s-sparse}}(x_1,\dots, x_n) \\
        =& \left\{
            \pi\in \cS_n: x_{\pi(1)}^T u \leq x^T_{\pi(2)} u \leq \dots \leq x_{\pi(n)}^T u
            \text{ for some s-sparse } u\in \bbR^M
        \right\},
        \end{split}
    \end{equation}
    where $\cS_n$ contains all permutations for $n$ objects. By union bound and the preserving order property for isotonic regression, we have for any $j\in [M]$,
    \begin{equation}
        \begin{split}
            &P\left(
            \sup_{\bfv \in \bfv} \parallel \iso_{\bfv} (\X_{-0,j}) - \iso_{\bfv}(f_j^{*}(\X_{-T}u_j^{*})) )\parallel_2^2 > t
        \right) \\
        \leq &
        \left| 
        \cS_{T,M}^{s_j\text{-sparse}}(X_0,\dots, X_{T-1}) 
        \right|\cdot
        P\left( \parallel \iso_{\bfv} (\X_{-0,j}) - \iso_{\bfv}(f_j^{*}(\X_{-T}u_j^{*})) )\parallel_2^2 > t \right)\\
        \leq& T^{2s_j-1} M^{s_j} \cdot
        P\left( \parallel \iso_{\bfv} (\X_{-0,j}) - \iso_{\bfv}(f_j^{*}(\X_{-T}u_j^{*})) )\parallel_2^2 > t \right),\\
        \end{split}
    \end{equation}
    where the last line comes from \cref{lemma: permutation}. Therefore, combining \cref{lemma: 2norm bound}, we know that with probability $1-2\gamma$,
    \begin{equation}
        \begin{split}
            &\sup_{\bfv \in \bfv} \parallel \iso_{\bfv} (\X_{-0,j}) - \iso_{\bfv}(f_j^{*}(\X_{-T}u_j^{*})) )\parallel_2^2\\
            \leq & 16 T^{\frac{1}{3}} \left[\left(2\sqrt{2\sigma_j^2 \log \frac{2T}{\gamma/(T^{2s_j-1}p^{s_j})} } + 2 L_j M_x \sqrt{s_j^{*}} \right)
            \sigma_j^2 \log( \frac{T(T+1)}{\gamma/(T^{2s_j-1}M^{s_j})})
        \right]^{\frac{2}{3}} \\
        &+ 4 \sigma_j^2 \log
        \left(\frac{T(T+1)}{\gamma/(T^{2s_j-1}M^{s_j})}\right),
        \end{split}
    \end{equation}
    by the fact that $\sqrt{a^2 + b^2} \leq a + b$ for any $a,b>0$, we could conclude that
    \begin{equation}
        \begin{split}
            &\sup_{\bfv\in \bfv} \|\iso_{\bfv} (f_j^{*}(\X_{-T}u_j^{*})) - \iso_{\bfv}(\X_{-0,j}) \|_2\\
            \leq & 4 T^{\frac{1}{6}} \left[
                \left(2\sqrt{2\sigma_j^2 \log \frac{2T^{2s_j}p^{s_j}}{\gamma} } + 2 L_j M_x \sqrt{s_j^{*}} \right)
                \sigma_j^2 \log( \frac{T^{2s_j}(T+1)M^{s_j}}{\gamma})
            \right]^{\frac{1}{3}} \\
            &+ 2\sigma_j \sqrt{\log\left(
                \frac{T^{2s_j}(T+1) M^{s_j}}{\gamma}
            \right)}
        \end{split}
    \end{equation}
    with probability at least $1-2\gamma$.
\end{proof}

\subsection{Proof of \cref{lemma: good_init}}
\begin{proof}
    As is defined, 
    $u_j^{(0)} = \frac{\Phi_{s_j}(\tilde{u}_j^{(0)})}{\parallel \Phi_{s_j}(\tilde{u}_j^{(0)}) \parallel_2}.
    $
    Hence \cref{eq: good_init} is equivalent to $\langle \Phi_{s_j}(\tilde{u}_j^{(0)}) , u_j^*\rangle >0 $. We conquer it by splitting it into two parts:
    \begin{equation}
        \langle \Phi_{s_j}(\tilde{u}_j^{(0)}) , u_j^*\rangle = \langle \tilde{u}_j^{(0)} , u_j^*\rangle + \langle \Phi_{s_j}(\tilde{u}_j^{(0)})  -  \tilde{u}_j^{(0)} , u_j^*\rangle.
        \label{eq: split_init}
    \end{equation}
    For the first term in the right-hand side, we have
    \begin{equation}
        \begin{split}
            \langle \tilde{u}_j^{(0)} , u_j^*\rangle  
            =& \langle \frac{1}{T}\X_{-T}^T(\X_{-0,j} - \overbar{\X_{-0,j}}\mathbf{1}_T), u_j^*
            \rangle\\
            =&\frac{1}{T} \langle f_j^*(\X_{-T}^Tu_j^*) - \overbar{f_j^*(\X_{-T}^T u_j^*)}\cdot \mathbf{1}_T,\X_{-T} u_j^* \rangle
            +
            \langle \frac{1}{T}\X_{-T}^T(\Z_{j} - \overbar{\Z_j}\cdot \mathbf{1}_T), u_j^* \rangle \\ 
            \geq &\frac{1}{T}L_j^{-1} \parallel f_j^*(\X_{-T}^Tu_j^*) - \overbar{f_j^*(\X_{-T}^Tu_j^*)}\cdot \mathbf{1}_T \parallel_2^2 -
            \sqrt{s_j^*} \parallel \frac{1}{T}\X_{-T}^T (\Z_j - \overbar{\Z_j}\mathbf{1}_T) \parallel_{\infty}.
        \end{split}
    \end{equation}
    Note that in the above last line, the first term came from \cref{lemma: lip_bound}, with $u$ taken as $-u_j^*$; the second term is derived from the duality equality and $\parallel u_j^* \parallel_1 \leq \sqrt{s_j^*}$.

    As for the second term of \cref{eq: split_init}, with the fact that $\langle \tilde{u}_j^{(0)} - \Phi_{s_j}(\tilde{u}_j^{(0)}) , \Phi_{s_j}(\tilde{u}_j^{(0)})\rangle = 0$, we have
    \begin{equation}
            \begin{split}
                &\langle \Phi_{s_j}(\tilde{u}_j^{(0)})  -  \tilde{u}_j^{(0)} , u_j^*\rangle\\
                 =&  -\langle  \tilde{u}_j^{(0)}- \Phi_{s_j}(\tilde{u}_j^{(0)})  , u_j^*- \Phi_{s_j}(\tilde{u}_j^{(0)})\rangle\\
                 \geq& - \frac{\sqrt{s_j^*}}{2\sqrt{s_j}} \parallel u_j^* - \Phi_{s_j}(\tilde{u}_j^{(0)}) \parallel_2^2\\
                 = & -\frac{\sqrt{s_j^*}}{2\sqrt{s_j}} \left(
                    1+\parallel \Phi_{s_j}(\tilde{u}_j^{(0)})\parallel_2^2 
                \right)
                +\sqrt{\frac{s_j^*}{s_j}}
                \langle u_j^*, \Phi_{s_j}(\tilde{u}_j^{(0)})\rangle.\\
            \end{split}
    \end{equation}
    Specifically, the norm of $s_j$-sparsity enforced vector $\Phi_{s_j}(\tilde{u}_j^{(0)})$ can be bounded by:
    \begin{equation}
        \begin{split}
            &\parallel \Phi_{s_j}(\tilde{u}_j^{(0)}) \parallel_2^2 \\
            \leq &2 \left\{
                \left\| \Phi_{s_j}\left(
                    \frac{1}{T} \X_{-T}^T (f_j^*(\X_{-T} u_j^*) - \overbar{f_j^*(\X_{-T} u_j^*)} \cdot \mathbf{1}_T)
                \right) \right\|_2^2 +
                \left\| \frac{1}{T}\Phi_{s_j}\left(
                    \X_{-T}^T (\Z_j -\overbar{\Z_j}\mathbf{1}_T)
                \right)\right\|_2^2
            \right \}\\
            \leq&2\cdot\frac{2}{T^2} \|f_j^*(\X_{-T} u_j^*) - \overbar{f_j^*(\X_{-T} u_j^*)}\mathbf{1}_T \|_2^2 \cdot \| \X_{-T}^T \mathbf{1}_{A_{s_j}} \|_2^2 +2s_j \| \frac{1}{T}\X_{-T}^T (\Z_j - \overbar{\Z_j}\mathbf{1}_T \|_{\infty}^2\\
            \leq & \frac{4\beta s_j}{T^2} \|f_j^*(\X_{-T} u_j^*) - \overbar{f_j^*(\X_{-T} u_j^*)}\mathbf{1}_T \|_2^2
            +
            2s_j \parallel \frac{1}{T} \X_{-T}^T (\Z_j - \overbar{\Z_j}\mathbf{1}_T) \parallel^2_{\infty}.\\
        \end{split}
    \end{equation}
    Putting things together, we have
    \begin{equation}
        \begin{split}
            &(1 - \sqrt{\frac{s_j^*}{s_j}})\langle \Phi_{s_j}(\tilde{u}_j^{(0)}) , u_j^*\rangle \\
            \geq & (\frac{1}{L_j } -\frac{2\beta\sqrt{s_js_j^*}}{T})
            \cdot 
            \frac{1}{T} \| f_j^*(\X_{-T}^T u_j^* ) - \overbar{f_j^*(\X_{-T} u_j^*)}\mathbf{1}_T \|_2^2\\
            &-\sqrt{s_j s_j^*} \|\frac{1}{T} \X_{-T}^T (\Z_j - \overbar{\Z_j}\mathbf{1}_T) \|_{\infty}^2 
             - \sqrt{s_j^*} \cdot\|\frac{1}{T} \X_{-T}^T(\Z_j - \overbar{\Z_j}\mathbf{1}_T)\|_{\infty} -\frac{\sqrt{s_j^*}}{2\sqrt{s_j}}.
        \end{split}
        \label{eq: init_together}
    \end{equation}
    Note that for any $u \in \bbR^p$ with norm 1, by the optimality of $\iso_{\X_{-T} u}(\cdot)$ and the identifiability condition \cref{eq: identifiability}, we have
    \begin{equation}
        \begin{split}
            &\frac{1}{T}\| f_j^*(\X_{-T}^T u_j^* ) - \overbar{f_j^*(\X_{-T} u_j^*)}\mathbf{1}_T \|_2^2\\
             &\geq\frac{1}{T} \sup_{u\in \bbR^p, \| u \|_2 =1}\| f_j^*(\X_{-T}^T u_j^* ) - \iso_{\X_{-T}u} (f_j^*(\X_{-T}^T u_j^* )) \|_2^2\\
            & \geq \alpha_j  \sup_{u\in \bbR^p, \| u \|_2 =1}\| u - u_j^* \|_2^2 - \epsilon_j^2 \geq 4\alpha_j - \epsilon_j^2,
        \end{split}
        \label{eq: lb_2norm}
    \end{equation}
    where the last inequality can be reached when we take a $u = -u_j^*$. 

    Plug \cref{eq: lb_2norm} into \cref{eq: init_together}, we have that $(1 - \sqrt{\frac{s_j^*}{s_j}})\langle \Phi_{s_j}(\tilde{u}_j^{(0)}) , u_j^*\rangle >0$ as long as 
    \begin{equation}
        \|\frac{1}{T} \X_{-T}^T(\Z_j - \overbar{\Z_j}\mathbf{1}_T)\|_{\infty} < U_j^{+} := \frac{1}{2}\left(
            \sqrt{\frac{1}{s_j} + 4c_j^{+}}-\sqrt{\frac{1}{s_j}}
        \right),
    \end{equation}
    where 
    \begin{equation}
        c_j^{+} := \left(\frac{1}{L_j\sqrt{s_j s_j^*}}- \frac{2\beta}{T}\right) (4\alpha_j - \epsilon_j^2)
        - \frac{1}{2s_j}.
    \end{equation}
    To assure $U_j^{+} >0$, we need $c_j^{+} >0$. $ 4\alpha_j - \epsilon_j >0$, we have 
    \begin{equation}
        s_j > s_j^* \left(\frac{L_j}{2(4\alpha_j -\epsilon_j^2)} + O(\frac{\beta L_j^3}{T})\right)^2.
    \end{equation}

    By union bound, data boundedness and sub-Gaussian tail, we have
    \begin{equation}
        \begin{split}
            &P\left(
                \| \frac{1}{T} \X_{-T}^T (\Z_j - \overbar{\Z_j}\mathbf{1}_T)\|_{\infty} \geq t 
            \right)\\
            \leq & P\left(
                \frac{1}{T} \| \X_{-T}^T \Z_j \|_{\infty} \geq \frac{t}{2}
            \right) +
            P\left(
                \frac{1}{T} \| \X_{-T}^T \overbar{\Z_j}\mathbf{1}_T) \|_{\infty} \geq \frac{t}{2}
            \right)\\
            \leq & 2(M+1) \exp(-\frac{Tt^2}{8\sigma_j^2 M_x}).
        \end{split}
    \end{equation}
    Hence, with probability 
    \begin{equation}
        1- 2(M+1)\exp(-\frac{T {U_j^{+}}^2}{8\sigma_j^2 M_x}), 
    \end{equation}
    we have $\langle u_j^{(0)} , u_j^* \rangle >0$, when $s_j > s_j^*\cdot \max \left\{ 1, \left( \frac{L_j}{2(4\alpha_j - \epsilon_j^2)} + O(\frac{\beta L_j^3}{T}) \right)^2\right\}$.
\end{proof}

\subsection{Proof of \cref{prop: prediction} }
\begin{proof}
    By Cauchy inequality, the definition of isotonic regression and its contractive property , for any $j\in [M]$, let $\hat{u}_j = u_j^{(K)}$,
    where $K \geq \frac{\log(\frac{R_j^2}{2})}{\log(\theta_j)}$, where $R_j$ and $\theta_j$ are defined in \cref{thm: coefficient conv}. Then we have
    \begin{equation}
        \begin{split}
        &\frac{1}{\sqrt{T}}\left\|
        \iso_{\X_{-T} \hat{u}_j} (X_{-0,j}) - f_j^*(\X_{-T} u_j^*) \right\|_2 \\
        \leq&
        \frac{1}{\sqrt{T}}\left\|
        \iso_{\X_{-T} \hat{u}_j} (X_{-0,j}) - \iso_{\X_{-T} \hat{u}_j}\left(f_j^*(\X_{-T} u_j^*)\right) \right\|_2 \\
        &+ \frac{1}{\sqrt{T}}\left\|
        \iso_{\X_{-T} \hat{u}_j}\left(f_j^*(\X_{-T} u_j^*)\right) - f_j^*(\X_{-T} u_j^*) \right\|_2\\
        \leq & \frac{1}{\sqrt{T}}\left\|
        \iso_{\X_{-T} \hat{u}_j} (X_{-0,j}) - \iso_{\X_{-T} \hat{u}_j}\left(f_j^*(\X_{-T} u_j^*)\right) \right\|_2 \\
        &+ \frac{1}{\sqrt{T}}\left\|
        f_j^*(\X_{-T} \hat{u}_j) - f_j^*(\X_{-T} u_j^*) \right\|_2\\
        \leq & \frac{1}{\sqrt{T}}\left\|
        \iso_{\X_{-T} \hat{u}_j} (X_{-0,j}) - \iso_{\X_{-T} \hat{u}_j}\left(f_j^*(\X_{-T} u_j^*)\right) \right\|_2 + 
        L_j\sqrt{\beta}\|
        \hat{u}_j -  u_j^* \|_2\\
        \leq& O_p(T^{-\frac{1}{3}}\log T),
        \end{split}
    \end{equation}
    where the last line comes directly from \cref{lemma: 2norm bound} and \cref{thm: coefficient conv}.
\end{proof}

\subsection{Other Supporting Lemmas}
\begin{lemma} (Lemma 6 in \cite{dai2021convergence})
    For any vector $u \in \bbR^M$,
    \begin{equation}
        \langle
        \X_{-T} u^*, f^*(\X_{-T}u^*) - \iso_{\bfx u}(f^*(\X_{-T}u^*)) 
        \rangle
        \geq 
        L^{-1} \parallel f^*(\X_{-T}u^*) - \iso_{\bfx u}(f^*(\X_{-T}u^*)) 
        \parallel_2^2,
    \end{equation}
    where $f^*$ is an L-Lipschitz monotone non-decreasing function.
    \label{lemma: lip_bound}

\end{lemma}

\begin{lemma}
    For any $j\in [M]$ , and an index set $I \subset [T]$, we have 
    \begin{equation}
        P\left(\sqrt{|I|} \cdot |\overline{(\Z_j)_{I}}| >t\right) \leq 2 \exp(-\frac{t^2}{2\sigma_j^2}),
    \end{equation}
    where $|I|$ denotes the cardinality of $I$.
\end{lemma}
\begin{proof}
    To show the tail bound of the martingale difference mean, we first bound the moment generating function. For any $\lambda>0$, we have
    \begin{equation}
        \begin{split}
            \bbE\left[\exp\left(\lambda({|I|} \cdot\overline{(\Z_j)_{I}})\right)
        \right] &= \bbE \left[ 
            \exp\left(
                \lambda\sum_{t\in I} Z_{t,j}
            \right)
        \right]\\
            & = \bbE\left\{\bbE \left[ 
                \bbE\left(
                    \bbE \left(\exp(
                    \lambda\sum_{t\in I} Z_{t,j})
                    | \mathcal{F}_{T-1}
                \right)| \mathcal{F}_{T-2}
                \right)\dots|\mathcal{F}_0
            \right] \right\}
            \\
            &\leq \exp\left(\frac{\lambda^2 |I|\sigma_j^2}{2}\right).
        \end{split}
    \end{equation}
    Then by tail bound of a sub-Gaussian variable, we finish the proof.
\end{proof}

\begin{definition}(Seminorm)
    Let $\mathcal{V}$ be a vector space over real numbers in $\bbR$. A map $\|\cdot \|: \mathcal{V}\mapsto \bbR$ is called a seminorm if it satisfies the following two conditions:
    \begin{enumerate}
        \item Subadditivity (Triangle inequality): $\|x+y \|\leq \|x\| +\|y\|, ~ \forall x, y \in \mathcal{V}$;
        \item Homogeneity: $\| c x\| = |c| \cdot \|x\|, ~ \forall c\in \bbR, x\in \mathcal{V}$.
    \end{enumerate}
\end{definition}

\begin{remark}
    If a seminorm $\|\cdot\|$ separates points, \ie $\|x\| = 0$ implies $x=0$ for any $x\in \mathcal{V}$, it is also a norm. Define a mapping $\|\cdot \|_{mean}: \mathcal{V}\mapsto \bbR$ such that $\|x\|_{mean} = |\bar{x}|$, we could find that $\|\cdot\|_{mean}$ is not a norm, since for $x= [1,-1]^T$, by definition $\|x\|_{mean} = 0$. However, it's a seminorm, by the fact that for any $x,y \in \mathcal{V}$ and any $c\in \bbR$, we have
    \begin{enumerate}
        \item $\|x+y\|_{mean} = |\overline{x+y}| = |\bar{x}+\bar{y}| \leq |\bar{x}|+ |\bar{y}| = \|x\|_{mean}+
        \|y\|_{mean}$;
        \item $\| cx\|_{mean} = |\overline{cx}| = |c|\cdot\| x\|_{mean} $.
    \end{enumerate}
\end{remark}

\begin{lemma} (Lemma 1 in \citep{yang2017contraction})
    For any $n\in \mathbb{T}^{+}$ and seminorm $\|\cdot \|$ on the vector space $\bbR^n$, the isotonic projection is contractive with respect to $\|\cdot\|$, \ie
    \begin{equation}
        \| iso(x) - iso(y)\| \leq \|x - y\| ~\text{for all } x,y\in \bbR^n,
    \end{equation}
    as long as the seminorm $\|\cdot\|$ is invariant to permutations of the entries of the vector, that is, for any vector $x\in \bbR^n$ and permutation $\pi$ on $\{1,\cdots,n\},$
    \begin{equation}
        \|x\| = \|x_{\pi}\|.
    \end{equation}
    \label{lemma: contractive}
\end{lemma}

\begin{corollary}
    \label{cor: contract}
     With regard to the infinity norm $\|\cdot \|_{\infty}$, one-norm $\|\cdot\|_1$, 2-norm $\|\cdot\|_2$, as well as the mean seminorm $\|\cdot \|_{mean}$, the isotonic projection is contractive.
\end{corollary}

\begin{lemma}\label{lemma: hard_thres}
    (Lemma 1 in \citep{liu2018hard})
    For any $v\in \bbR^M$ and $s^*$-sparse $\omega \in \bbR^M$, it is true for the $s$-sparse hard-threshold operator $\Phi_s(\cdot)$ that
    \begin{equation}
        \frac{\langle v - \Phi_s(v), \omega -\Phi_s(v)\rangle}{\parallel \omega - \Phi_s(v) \parallel_2^2} \leq \frac{\sqrt{s^*}}{2\sqrt s}.
    \end{equation}
\end{lemma}

\begin{lemma}\cite{cover1967}
    Given a set of points $\{x_1,\dots, x_n\} \subset \bbR^M$ with sample size $n$ and dimension $M$, for any sparsity level $s\leq M$, define the set
    \begin{equation}
        \begin{split}
            &\cS_{n,M}^{\text{s-sparse}}(x_1,\dots, x_n) \\
        =& \left\{
            \pi\in \cS_n: x_{\pi(1)}^T u \leq x^T_{\pi(2)} u \leq \dots \leq x_{\pi(n)}^T u
            \text{ for some s-sparse } u\in \bbR^M
        \right\},
        \end{split}
    \end{equation}
    where $\cS_n$ contains all permutations for $n$ objects, \ie all bijections from the set $\{1,\dots, n\}$ onto itself.
    Then the cardinality of the set $\cS_{n,M}^{\text{s-sparse}}$ can be bounded as
    \begin{equation}
        \left|
        \cS_{n,M}^{\text{s-sparse}} (x_1,\dots, x_n)
        \right|
        \leq n^{2s-1} M^s.
    \end{equation}
    \label{lemma: permutation}
\end{lemma}

\end{document}